%% file: main.tex
\documentclass{article}

    \usepackage[final]{neurips_2019}

\usepackage[utf8]{inputenc} 
\usepackage[T1]{fontenc}    
\usepackage{hyperref}       
\usepackage{url}            
\usepackage{booktabs}       
\usepackage{amsfonts}       
\usepackage{nicefrac}       
\usepackage{microtype}      

\usepackage{amsmath}
\usepackage{amsthm}
\usepackage{amssymb}
\usepackage{enumitem}

\usepackage{microtype}
\usepackage{graphicx}
\usepackage{caption}
\usepackage{subcaption}
\usepackage{dblfloatfix}

\usepackage{tikz}
\tikzstyle{dist} = [rectangle, text centered, draw=black]
\tikzstyle{approx_dist} = [rectangle, text centered, draw=black, fill=orange!30]
\tikzstyle{action} = [rectangle, rounded corners, text centered, draw=black, fill=green!30]

\usepackage[colorinlistoftodos]{todonotes}

\newtheorem{theorem}{Theorem}
\newtheorem{lemma}{Lemma}

\newtheorem{corollary}{Corollary}

\newtheorem{assumption}{Assumption}

\title{Thompson Sampling with Approximate Inference}

\author{%
  My Phan \\
  College of Information and Computer Science\\
  University of Massachusetts\\
  Amherst, MA\\
  \texttt{myphan@cs.umass.edu} \\
   \And
  Yasin Abbasi-Yadkori \\
  VinAI \\
  Hanoi, Vietnam\\
 \texttt{yasin.abbasi@gmail.com} \\
  \And
  Justin Domke\\
  College of Information and Computer Science\\
  University of Massachusetts\\
  Amherst, MA\\
  \texttt{domke@cs.umass.edu} \\
}

\begin{document}

\maketitle

\begin{abstract}

We study the effects of approximate inference on the performance of Thompson sampling in the $k$-armed bandit problems. Thompson sampling is a successful algorithm for online decision-making but requires posterior inference, which often must be approximated in practice. We show that even small constant inference error (in $\alpha$-divergence) can lead to poor performance (linear regret) due to under-exploration (for $\alpha<1$) or over-exploration (for $\alpha>0$) by the approximation. While for $\alpha > 0$ this is unavoidable, for $\alpha \leq 0$ the regret can be improved by adding a small amount of forced exploration even when the inference error is a large constant.
\end{abstract}

\input{introduction_neurips}
\input{problem_formulation_neurips}

\input{small_motivating_example_neurips}
\input{alpha_larger_than_0_neurips}

\input{alpha_less_than_1_neurips}
\input{experiments_neurips}
\input{related_works}
\input{conclusion.tex}

\subsubsection*{Acknowledgments}
We thank Huy Le for providing the proof of Lemma 9. 
\bibliography{thompson}
\bibliographystyle{icml2019}
\clearpage
\input{appendix}

\end{document}

%% file: introduction_neurips.tex

\section{Introduction}


The stochastic $k$-armed bandit problem is a sequential decision making problem where at each time-step $t$, a learning agent chooses an action (arm) among $k$ possible actions and observes a random reward. Thompson sampling \citep{MAL-070} is a popular approach in bandit problems based on sampling from a posterior in each round. It has been shown to have good performance both in term of frequentist regret and Bayesian regret for the $k$-armed bandit problem under certain conditions.

This paper investigates Thompson sampling when only an \emph{approximate} posterior is available. This is motivated by the fact that in complex models, approximate inference methods such as Markov Chain Monte Carlo or Variational Inference must be used. Along this line, \citet{NIPS2017_6918} propose a novel inference method -- Ensemble sampling -- and analyze its regret for linear contextual bandits. To the best of our knowledge this is the most closely related theoretical analysis of Thompson sampling with approximate inference.


This paper analyzes the regret of Thompson sampling with approximate inference. Rather than considering a particular inference algorithm, we parameterize the error using the $\alpha$-divergence, a typical measure of inference accuracy. 
Our contributions are as follows:

\begin{itemize}

	\item {\bf Even small inference errors can lead to linear regret with naive Thompson sampling.} Given any error threshold $\epsilon>0$ and any $\alpha$ we show that approximate posteriors with error at most $\epsilon$ in $\alpha$-divergence at all times can result in linear regret (both frequentist and Bayesian). For $\alpha > 0$ and for any reasonable prior, we show linear regret due to over-exploration by the approximation (Theorem~\ref{thm:alpha_larger_than_0}, Corrolary~\ref{col:alpha_larger_than_0}).  For $\alpha < 1$ and for priors satisfying certain conditions, we show linear regret due to under-exploration by the approximation, which prevents the posterior from concentrating (Theorem~\ref{thm:alpha_less_than_1_example}, Corrolary~\ref{col:alpha_less_than_1_example}).
	
	
	\item {\bf Forced exploration can restore sub-linear regret.} For $\alpha \leq 0$ we show that adding forced exploration to Thompson sampling can make the posterior concentrate and restore sub-linear regret (Theorem~\ref{thm:uniform_sublinear_regret}) even when the error threshold is a very large constant. We illustrate this effect by showing that the performances of Ensemble sampling \citep{NIPS2017_6918} and mean-field Variation Inference \citep{doi:10.1080/01621459.2017.1285773} can be improved in this way either theoretically (Section ~\ref{subsec:uniform}) or in simulations (Section~\ref{sec:experiments}).
\end{itemize}

%% file: problem_formulation_neurips.tex
\section{Background and Notations.}
\label{sec:problem_formulation}
\subsection{The $k$-armed Bandit Problem.}

We consider the $k$-armed bandit problem parameterized by the mean reward vector $m = (m_1, ..., m_k) \in \mathcal{R}^k$, where $m_i^*$  denotes the mean reward of arm (action) $i$. At each round $t$, the learner chooses an action $A_t$ and observes the outcome $Y_t$ which, conditioned on $A_t$, is independent of the history up to and not including time $t$, $H_{t-1} = (A_1, Y_1, ..., A_{t-1}, Y_{t-1})$. For a time horizon $T$, the goal of the algorithm $\pi$ is to maximize the expected cumulative reward up to time $T$.

Let $\Omega \subseteq \mathcal{R}^k$ be the domain of the mean and $\Omega_i \subseteq \Omega$ denote the region where the $i$th arm has the largest mean.  Let the function $A^*: \Omega \rightarrow \{a_1, ..., a_k\}$ denoting the best action be defined as: $A^*(m) = i \text{ if }  m \in \Omega_i$. 

In the frequentist setting we assume that there exists a true mean $m^*$ which is fixed and unknown to the learner. Therefore, a policy $\pi^*$ that always chooses $A^*(m^*)$ will get the highest reward. The performance of policy $\pi$ is measured by its expected regret compared to an optimal policy $\pi^*$, which  is defined as: 
\begin{align}
\label{eq:regret}
\mathrm{Regret}(T, \pi, m^*) 
=Tm_{A^*(m^*)}^*  -  \mathbb{E} \sum_{t=1}^T  m_{A_t}^* \;.
\end{align}

 On the other hand, in the Bayesian setting, an agent expresses her beliefs about the mean vector in terms of a prior $\Pi_0$, and therefore, the mean is treated as a random variable $M = (M_1,..., M_k)$ distributed according to the prior $\Pi_0$. The Bayesian regret is the expectation of the regret under the prior of parameter $M$:
\begin{align}
\mathrm{BayesRegret}(T, \pi) = \mathbb{E}_{\Pi_0} \mathrm{Regret}(T,\pi, M) \;.
\end{align}

\subsection{Thompson Sampling with Approximate Inference}
In the frequentist setting, in order to perform Thompson sampling we define a prior which is only used in the algorithm. On the other hand, in the Bayesian setting the prior is given. 

Let $\Pi_t$ be the posterior distribution of $M|H_{t-1}$ with density function $\pi_t(m)$. Thompson sampling obtains a sample $\widehat{m}$ from $\Pi_t$ and then selects arm $A_t$ as follow: $A_t = i \text{ if } \widehat{m} \in \Omega_i$. 
In each round, we assume an approximate sampling method is available that generates sample from an approximate distribution $Q_t$. We use $q_t$ to denote the density function of $Q_t$. 

Popular approximate sampling methods include 
Markov Chain Monte Carlo (MCMC) \citep{Andrieu2003},  Sequential Monte Carlo \citep{Doucet11atutorial} and
Variational Inference (VI) \citep{doi:10.1080/01621459.2017.1285773}. There are packages that conveniently implement VI and MCMC methods, such as Stan \citep{stan:2017}, Edward \citep{tran2016edward}, PyMC \citep{Salvatier2016ProbabilisticPI} and infer.NET \citep{InferNET18}. 

To provide a general analysis of approximate sampling methods, we will use the $\alpha$-divergence (Section~\ref{sec:alpha_divergence}) to quantify the distance between the posterior $\Pi_t$ and the approximation $Q_t$. 
\subsection{The Alpha Divergence}
\label{sec:alpha_divergence}
The $\alpha$-divergence between two distributions $P$ and $Q$ with density functions $p(x)$ and $q(x)$ is defined as:
\begin{align}
\label{def:alpha_divergence_long}
D_{\alpha}(P,Q) = \frac{1- \int p(x)^{\alpha}q(x)^{1-\alpha} dx}{\alpha(1-\alpha)}. 
\end{align}

$\alpha$-divergence generalizes many divergences, including $KL(Q,P)$ ($\alpha \rightarrow 0$), $KL(P, Q)$ ($\alpha \rightarrow 1$), Hellinger distance ($\alpha = 0.5$) and $\chi^2$ divergence ($\alpha = 2$) and is a common way to measure errors in inference methods. MCMC errors are measured by the Total Variation distance, which can be upper bounded by the KL divergence using Pinsker's inequality ($\alpha = 0$ or $\alpha = 1$). Variational Inference tries to minimize the reverse KL divergence (information projection) between the target distribution and the approximation ($\alpha=0$). Ensemble sampling \citep{NIPS2017_6918} provides error guarantees using reverse KL divergence ($\alpha=0$).
Expectation Propagation tries to minimize the KL divergence ($\alpha=1$) and $\chi^2$ Variational Inference tries to minimize the $\chi^2$ divergence ($\alpha=2$). 
\begin{figure}[h]
	\centering
	\includegraphics[width=\textwidth]{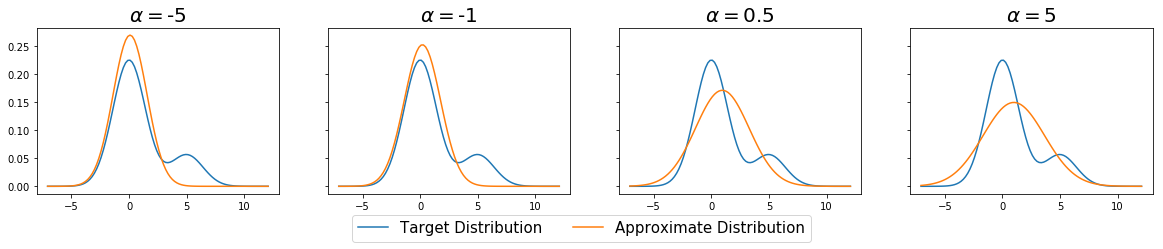}
	\caption{\label{fig:alpha_divergence} The Gaussian $Q$ which minimizes $D_{\alpha}(P,Q)$ for different values of $\alpha$ where the target distribution $P$ is a mixture of two Gaussians. Based on Figure 1 from \citep{divergence-measures-and-message-passing}}
\end{figure}

When $\alpha$ is small, the approximation fits the posterior's dominant mode. When $\alpha$ is large, the approximation covers the posterior's entire support \citep{divergence-measures-and-message-passing} as illustrated in Figure~\ref{fig:alpha_divergence}. 
Therefore changing $\alpha$ will affect the exploration-exploitation trade-off in bandit problems. 

\subsection{Problem Statement. }

\textbf{Problem Statement.} For the $k$-armed bandit problem, given $\alpha$ and $\epsilon>0$, if at all time-steps $t$ we sample from an approximate distribution $Q_t$ such that $D_{\alpha}(\Pi_t, Q_t) < \epsilon$, will the regret be sub-linear in $t$?

%% file: small_motivating_example_neurips.tex

\section{Motivating Example}
\label{sec:motivating_example}

In this section we present a simple example to show the effects of inference errors on the frequentist regret.


{\bf Example.} Consider a 2-armed bandit problem where the reward distributions are $\mathrm{Norm}(0.6,0.2^2)$ and $\mathrm{Norm}(0.5, 0.2^2)$ for arm $1$ and $2$ respectively. The prior $\Pi_0$ is $\mathrm{Norm}\left (\mu_0^T, 0.5^2 I \right)$ where $\mu_0 = [0.1, 0.9]$ is the vector of prior means of arm $1$ and $2$ respectively, and $I$ denotes the identity matrix.  

\begin{figure}[ht]
	\centering
	\begin{subfigure}[t]{0.45\textwidth}
		\centering
		\includegraphics[width = \textwidth]{fig/motivating_example_edited.png}
		\caption{\label{fig:motivating_example}  Over-dispersed (approximation $Q_t$) and under-dispersed sampling (approximation $Z_t$) yield different posteriors after $T=100$ time-steps. $m_1$ and $m_2$ are the means of arms $1$ and $2$. $Q_t$ picks arm $2$ more often than exact Thompson sampling and $Z_t$ mostly picks arm $2$. The posteriors of exact Thompson sampling and $Q_t$ concentrate mostly in the region where $m_1>m_2$ while $Z_t$'s spans both regions.}
	\end{subfigure}
	\qquad
	\begin{subfigure}[t]{0.45\textwidth}
		\centering
		\includegraphics[width =\textwidth]{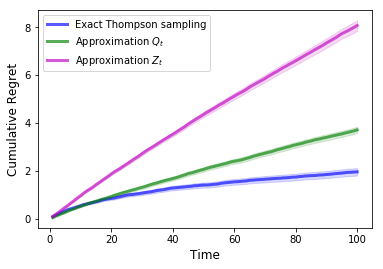}
		\caption{\label{fig:small_motivating_example_simulation}  The regret of sampling from the approximations $Q_t$  and $Z_t$  are both larger than that of exact Thompson sampling from the true posterior $\Pi_t$. Shaded regions show 95\% confidence intervals. }
	\end{subfigure}
	\caption{Approximation $Q_t$ (with high variance) and approximation $Z_t$ (with small variance) are defined in Section~\ref{sec:motivating_example} where $D_{1}(\Pi_t, Q_t) =2$ and $D_{0}(\Pi_t, Z_t) =1.5$. Arm $1$ is the true best arm. }
\end{figure}

Let $\Pi_t = \mathrm{Norm}(\mu_t, \Sigma_t)$ be the posterior at time $t$. Approximations $Q_t$ and $Z_t$ are calculated such that
$\mathrm{KL}(\Pi_t, Q_t)  =  2 \text{ and }\mathrm{KL}(Z_t, \Pi_t) = 1.5$ by multiplying the covariance $\Sigma_t$ by a constant: $Q_t=\mathrm{Norm} (\mu_t, 4.5^2\Sigma_t)$ and $Z_t = \mathrm{Norm} (\mu_t, 0.3^2\Sigma_t).$ The KL divergence between two Gaussian distributions is provided in Appendix~\ref{sec:KL_divergence_gaussian}. 

We perform the following simulations $1000$ times and plot the mean cumulative regret up to time $T=100$  in Figure~\ref{fig:small_motivating_example_simulation} using three different policies:
\begin{enumerate}[nosep]
	\item ({\bf Exact Thompson Sampling}) At each time-step $t$, sample from the true posterior $\Pi_t$.
	\item ({\bf Approximation $Q_t$}) At each time-step $t$, compute $Q_t$ from $\Pi_t$ and sample from $Q_t$. 
	\item ({\bf Approximation $Z_t$}) At each time-step $t$, compute $Z_t$ from $\Pi_t$ and sample from $Z_t$.
\end{enumerate}

The regrets of sampling from the approximations $Q_t$ and $Z_t$ are in both cases larger than that of exact Thompson sampling. Intuitively, the regret of $Q_t$ is larger because $Q_t$ explores more than the true posterior (Figure~\ref{fig:motivating_example}). 
In Section~\ref{sec:alpha_larger_than_0} we show that when $\alpha > 0$ the approximation can incur this type of error, leading to linear regret. On the other hand, the regret of $Z_t$ is larger because $Z_t$ explores less than the exact Thompson sampling algorithm and therefore commits to the sub-optimal arm (Figure~\ref{fig:motivating_example}). 
In Section~\ref{sec:alpha_less_than_1} we show that when $\alpha <1$ the approximation can change the posterior concentration rate, leading to linear regret. We also show that adding a uniform sampling step can help the posterior to concentrate when $\alpha \leq 0$, and make the regret sub-linear. 

%% file: alpha_larger_than_0_neurips.tex
\section{Regret Analysis When $\alpha > 0$}
\label{sec:alpha_larger_than_0}

In this section we analyze the regret when $\alpha > 0$.  Our result shows that  the approximate method might pick the sub-optimal arm with constant probability in every time-step, leading to linear regret.
\begin{theorem}[Frequentist Regret]
	\label{thm:alpha_larger_than_0}
	Let $\alpha >0$, the number of arms be $k = 2$ and $m^*_1>m^*_2$. Let $\Pi_0$ be a prior where $\mathbb{P}_{\Pi_0}(M_2 > M_1 ) > 0$. For any error threshold $\epsilon > 0$, there is a deterministic mapping $f(\Pi)$ such that for all $t\geq 0$:
	\begin{enumerate}[nosep]
		\item Sampling from $Q_t= f(\Pi_t)$ chooses arm $2$ with a constant probability.
			\item $D_{\alpha} (\Pi_t,Q_t) < \epsilon$.
	\end{enumerate}
	Therefore sampling from $Q_t$ for $T/10$ time-steps and using any policy for the remaining time-steps will cause linear frequentist regret.
\end{theorem}

Typically, approximate inference methods minimize divergences. Broadly speaking, this theorem shows that making a divergence a small constant, alone, is not enough to guarantee sub-linear regret. We do not mean to imply that low regret is {\em impossible} but simply that making an $\alpha$-divergence a small constant alone is not sufficient.

At every time-step, the mapping $f$ constructs the approximation $Q_t$ from the posterior $\Pi_t$ by moving probability mass from the region $\Omega_1$ where $m_1 > m_2$ to the region $\Omega_2$ where $m_2 > m_1$.  Then $Q_t$ will choose arm $2$ with a constant probability at every time-step. The constant average regret per time-step is discussed in Appendix~\ref{apx:average_regret_alpha_larger_than_0}.

Therefore, if we sample from $Q_t=f(\Pi_t)$ for $0.1 T$ time steps and use any policy in the remaining $0.9T$ time steps, we will still incur linear regret from the $0.1T$ time-steps. On the other hand, when $\alpha \leq 0$, we show  in Section~\ref{subsec:uniform} that sampling an arm uniformly at random for $\log{T}$ time-steps and sampling from an approximate distribution that satisfies the divergence constraint for $T - \log{T}$ time-steps will result in sub-linear regret. 

\citet{pmlr-v31-agrawal13a} show that the frequentist regret of exact Thompson sampling is $O(\sqrt{T})$ with Gaussian or Beta priors and bounded rewards. Theorem~\ref{thm:alpha_larger_than_0} implies that when the assumptions in \citep{pmlr-v31-agrawal13a} are satisfied but there is a small constant inference error at every time-step, the regret is no longer guaranteed to be sub-linear.

If the assumption $m^*_1>m^*_2$ in Theorem~\ref{thm:alpha_larger_than_0} is satisfied with a non-zero probability $\left (\mathbb{P}_{\Pi_0}(M_1 > M_2) >0\right )$, the Bayesian regret will also be linear:
\begin{corollary}[Bayesian Regret]
	\label{col:alpha_larger_than_0}
	Let $\alpha >0$ and the number of arms be $k = 2$. Let $\Pi_0$ be a prior where $\mathbb{P}_{\Pi_0}(M_1 > M_2 ) > 0$ and $\mathbb{P}_{\Pi_0}(M_2>M_1)>0$. Then for any error threshold $\epsilon > 0$, there is a deterministic mapping $f(\Pi)$ such that for all $t\geq 0$ the two statements in Theorem~\ref{thm:alpha_larger_than_0} hold. 
	
		Therefore sampling from $Q_t$ for $T/10$ time-steps and using any policy for the remaining time-steps will cause linear Bayesian regret.
\end{corollary}
\citet{JMLR:v17:14-087} prove that the Bayesian regret of Thompson sampling for $k$-armed bandits with sub-Gaussian rewards is $O(\sqrt{T})$. Corollary~\ref{col:alpha_larger_than_0} implies that even when the assumptions in \citet{JMLR:v17:14-087}  are satisfied, under certain conditions and with approximation errors, the regret is no longer guaranteed to be sub-linear.

%% file: alpha_less_than_1_neurips.tex
  \section{Regret Analysis When $\alpha < 1$}
\label{sec:alpha_less_than_1}
\input{alpha_less_than_1_example_neurips}
\input{uniform_neurips}

%% file: alpha_less_than_1_example_neurips.tex
In this section we analyze the regret when $\alpha < 1$. Our result shows that for any error threshold, if the posterior $\Pi_t$ places too much probability mass on the wrong arm then the approximation $Q_t$ is allowed to avoid the optimal arm. If the sub-optimal arms do not provide information about the arms'  ranking, the posterior $\Pi_{t+1}$ does not concentrate. Therefore $Q_{t+1}$ is also allowed to be close in $\alpha$-divergence while avoiding the optimal arm, leading to linear regret in the long term.

\begin{theorem}[Frequentist Regret]
	\label{thm:alpha_less_than_1_example}
	Let $\alpha <1$, the number of arms be $k = 2$   and  $m^*_1>m^*_2$.  Let $\Pi_0$ be a prior where  $M_2$ and $M_1 -M_2$ are independent. There is a deterministic mapping $f(\Pi)$ such that for all $t\geq 0$:
			\begin{enumerate}[nosep]
				\item Sampling from $Q_t=f(\Pi_t)$ chooses arm $2$ with probability $1$.
				\item For any $\epsilon >0$, there exists $ 0<z \leq 1$ such that if $\mathbb{P}_{\Pi_0}(M_2 > M_1)=z$ and arm $2$ is chosen at all times before $t$ then $D_{\alpha} (\Pi_t, Q_t) < \epsilon$ . 
				
				For any $0<z\leq 1$, there exists $ \epsilon >0$ such that if $\mathbb{P}_{\Pi_0}(M_2 > M_1)=z$ and arm $2$ is chosen at all times before $t$ then $D_{\alpha} (\Pi_t, Q_t) < \epsilon$. 
			\end{enumerate}
			Therefore sampling from $Q_t$ at all time-steps results in linear frequentist regret. 
%
%
%
\end{theorem}
We discuss why the above results are not immediately obvious. When $\alpha \rightarrow 0$, the $\alpha$-divergence becomes $\mathrm{KL}(Q_t, \Pi_t) $. We might believe that the regret should be sub-linear in this case because the posterior $\Pi_t$ becomes more concentrated, and so the total variation between $Q_t$ and $\Pi_t$ must decrease. For example, \citet{Ordentlich2004ADD} show the distribution-dependent Pinsker's inequality between $\mathrm{KL}(Q, P)$ and the total variation $\mathrm{TV}(P,Q)$ for discrete distributions $P$ and $Q$ as follows:
\begin{align}
\label{eq:Pinsker}
\mathrm{KL}(Q, P ) \geq \phi(P) \cdot \mathrm{TV}(P,Q)^2 \;.
\end{align}
Here, $\phi(P)$ is a quantity that will increase to infinity if $P$ becomes more concentrated. However, the algorithm in Theorem~\ref{thm:alpha_less_than_1_example} constructs an approximation distribution that never picks the optimal arm, so the posterior $\Pi_t$ can not concentrate and the regret is linear. The error threshold $\epsilon$ causing linear frequentist regret is correlated with the probability mass the prior places on the true best arm (Appendix~\ref{apx:prior_alpha_less_than_1}). 

 With some assumptions on the rewards, \citet{2013arXiv1311.0466G} show that the problem-dependent frequentist regret  is $O(\log T)$ for finitely-supported, correlated priors with $\pi_0(m^*) >0$.  \citet{DBLP:journals/corr/LiuL15c} study the prior-dependent frequentist regret of $2$-armed-and-$2$-models bandits, and show that with some smoothness assumptions on the reward likelihoods, the regret is $O(\sqrt{T/\mathbb{P}_{\Pi_0}(M_2>M_1)}$ if arm $1$ is the better arm.  
Theorem~\ref{thm:alpha_less_than_1_example} implies that when the assumptions in \citep{2013arXiv1311.0466G} or \citep{DBLP:journals/corr/LiuL15c} are satisfied, if $M_2$ and $M_1-M_2$ are independent and there are approximation errors, the regret is no longer guaranteed to be sub-linear. 

If the assumption $m^*_1 > m^*_2$ in Theorem~\ref{thm:alpha_less_than_1_example} is satisfied with a non-zero probability $\left (\mathbb{P}_{\Pi_0}(M_1 > M_2) >0\right )$, the Bayesian regret wil also be linear:
\begin{corollary}[Bayesian Regret]
		\label{col:alpha_less_than_1_example}
		Let $\alpha <1$ and the number of arms be $k = 2$. Let $\Pi_0$ be a prior where $\mathbb{P}_{\Pi_0}(M_1 > M_2) >0$ and $M_2$ and $M_1 -M_2$ are independent. There is a deterministic mapping $f(\Pi)$ such that for all $t\geq 0$ the 2 statements in Theorem~\ref{thm:alpha_less_than_1_example} hold. 
		
		Therefore sampling from $Q_t$ at all time-steps results in linear Bayesian regret. 
\end{corollary}
%

\citet{JMLR:v17:14-087} prove that the Bayesian regret of Thompson sampling for $k$-armed bandits with sub-Gaussian rewards is $O(\sqrt{T})$. Corollary~\ref{col:alpha_less_than_1_example} implies that even when the assumptions in \citet{JMLR:v17:14-087}  are satisfied, under certain conditions and with approximation errors, the regret is no longer guaranteed to be sub-linear.

We note that, unlike the case when $\alpha > 0$, if we use another policy  in $o(T)$ time-steps to make the posterior concentrate and sample from $Q_t$ for the remaining time-steps, the regret can be sub-linear. We provide a concrete algorithm in Section~\ref{subsec:uniform} for the case when $\alpha \leq 0$. 

%% file: uniform_neurips.tex
\subsection{Algorithms with Sub-linear Regret for $\alpha \leq 0$}
\label{subsec:uniform}
In the previous section, we see that when $\alpha< 1$, the approximation has linear regret because the posterior does not concentrate. In this section we show that when $\alpha \leq 0$, it is possible to achieve sub-linear regret even when $\epsilon$ is a very large constant by adding a simple exploration step to force the posterior to concentrate (the case of $\alpha >0$ cannot be improved according to Theorem~\ref{thm:alpha_larger_than_0}). We first look at the necessary and sufficient condition that will make the posterior concentrate, and then provide an algorithm that satisfies it. \cite{DBLP:journals/corr/Russo16} and \cite{NIPS2017_7122} both show the following result under different assumptions: 
\begin{lemma}[Lemma 14 from \cite{DBLP:journals/corr/Russo16}]
	\label{lem:posterior_concentration}
	Let $m^* \in \mathcal{R}^k$ be the true parameter and let $a^* = A^*(m^*)$ be the true best arm. If for all arms $i$, $\sum_{t=1}^\infty P(A_t = i|H_{t-1}) = \infty$,
	then 
	\begin{align}
	\lim_{t \rightarrow \infty} P(A^*(M) = a^*|H_{t-1}) = 1 \text{ with probability } 1\;.
	\end{align}
	If there exists arm $i$ such that $\sum_{t=1}^\infty P(A_t = i|H_{t-1}) < \infty$, then 
$\liminf_{t \rightarrow \infty} P(A^*(M) = i|H_{t-1}) > 0$ with probability $1$.
\end{lemma}
 \cite{DBLP:journals/corr/Russo16} make the following assumptions, which allow correlated priors:
\begin{assumption}
	\label{assumption:prior}
	Let the reward distributions be in the canonical one dimensional exponential family with the density: $p(y|m) = b(y)\exp(mT(y) - A(m))$
	where $b, T$ and $A$ are known function and $A(m)$ is assumed to be twice differentiable. The parameter space $\Omega = (\overline{m}, \underline{m})$ is a bounded open hyper-rectangle, the prior density is uniformly bounded with $0 < \inf_{m \in \Omega} \pi_0(m) < \sup_{m \in \Omega} \pi_0(m) <\infty$ and the log-partition function has bounded first derivative with $\sup_{\theta \in [\overline{m}, \underline{m} ]} |A'(m)| < \infty$.
\end{assumption}
\cite{NIPS2017_7122} make the following assumptions:
\begin{assumption}
	\label{assumption:prior2}
	Let the prior be an uncorrelated multivariate Gaussian. Let the reward distribution of arm $i$ be  $\mathrm{Norm}(m_i, \sigma^2)$ with a common known variance $\sigma^2$ but unknown mean $m_i$. 
\end{assumption}

Even though we consider the error in sampling from the posterior distribution, the regret is a result of choosing the wrong arm. We define $\overline{\Pi}_t$ as the posterior distribution of the best arm and $\overline{Q}_t$ as the approximation of $\overline{\Pi}_t$ with the density functions
\begin{align*}
\overline{\pi}_t(i) = P(A^* = i|H_{t-1}) \text{ and } \overline{q}_t(i)  = P(A_t = i|H_{t-1}).
\end{align*} 
We now define an algorithm where each arm will be chosen infinitely often, satisfying the condition of Lemma~\ref{lem:posterior_concentration}.
\begin{theorem}[Bayesian and Frequentist Regret]
	\label{thm:uniform_sublinear_regret}
	Consider the case when Assumption~\ref{assumption:prior} or \ref{assumption:prior2} is satisfied. Let $\alpha \leq 0$ and $p_t= o(1)$ be such that $\sum_{t=1}^\infty p_t = \infty$. For any number of arms $ k$, any prior $\Pi_0$ and any error threshold $\epsilon > 0$, the following algorithm has $o(T)$ frequentist regret: at every time-step $t$,
	\begin{itemize}[nosep]
		\item with probability $1-p_t$, sample from an approximate posterior $Q_t$ such that $D_{\alpha} (\overline{\Pi}_t, \overline{Q}_t) < \epsilon$, 
		\item with probability $p_t$, sample an arm uniformly at random.
	\end{itemize}
	Since the Bayesian regret is the expectation of the frequentist regret over the prior, for any prior if the frequentist regret is sub-linear at all points the Bayesian regret will be sub-linear. 
\end{theorem}

The following lemma shows that the error in choosing the arms is upper bounded by the error in choosing the parameters. Therefore whenever the condition $D_{\alpha} (\Pi_t, Q_t) <\epsilon $ is satisfied, the condition $D_{\alpha} (\overline{\Pi}_t, \overline{Q}_t)  < \epsilon$ will be satisfied and Theorem~\ref{thm:uniform_sublinear_regret} is applicable.
\begin{lemma}
	\label{lem:convert_distribution}
	\begin{align*}
	D_{\alpha} (\overline{\Pi}_t, \overline{Q}_t)  \leq D_{\alpha} (\Pi_t, Q_t) \,.  
	\end{align*}
\end{lemma}

We also note that we can achieve sub-linear regret even when $\epsilon$ is a very large constant. We revisit Eq.~\ref{eq:Pinsker} to provide the intuition: 
$\mathrm{KL}(Q, P ) \geq \phi(P) \cdot \mathrm{TV}(P,Q)^2$. Here, $\phi(P)$ is a quatity that will increase to infinity if $P$ becomes more concentrated. Hence, if $KL(\overline{Q}_t, \overline{\Pi}_t) < \epsilon$ for any constant $\epsilon$ and $\overline{\Pi}_t$ becomes concentrated, the total variation $\mathrm{TV}(\overline{Q_t}, \overline{\Pi}_t)$ will decrease. Therefore, $\overline{Q}_t$ will become concentrated, resulting in sub-linear regret. 

{\bf Application. }  \citet{NIPS2017_6918} propose an approximate sampling method  called Ensemble sampling where they maintain a set of $\mathcal{M}$ models to approximate the posterior and analyze its regret for the linear contextual bandits when $\mathcal{M}$  is $\Omega(\log(T))$. For the $k$-armed bandit problem and when $\mathcal{M}$  is $\Theta(\log(T))$, Ensemble sampling satisfies the condition $\mathrm{KL}(\overline{Q}_t, \overline{\Pi}_t) < \epsilon$ in Theorem~\ref{thm:uniform_sublinear_regret} with high probability. In this case, \citet{NIPS2017_6918} show a regret bound that scales linearly with $T$. We discuss in Appendix~\ref{sec:related_works_appendix} how to apply  Theorem~\ref{thm:uniform_sublinear_regret} to get sub-linear regret with Ensemble sampling when $\mathcal{M}$  is $\Theta(\log(T))$. 

%% file: experiments_neurips.tex
\section{Simulations}
\label{sec:experiments}
For each approximation method we repeat the following simulations for $1000$ times and plot the mean cumulative regret, using five different policies.
\begin{enumerate}[nosep]
	\item ({\bf Exact Thompson sampling}) Use exact posterior sampling to choose an action and update the posterior (for reference). 
	\item ({\bf Approximation method}) Use the approximation method to choose an action and update the posterior. We use the approximation naively without any modification. 
	\item ({\bf Forced Exploration}) With a probability (the exploration rate), choose an action uniformly at random and update the posterior. Otherwise, use the approximation method to choose an action and update the posterior. This is the method suggested by Thm. \ref{thm:uniform_sublinear_regret}.
	\item ({\bf Approximate Sample}) Use the approximation method to choose an action. Use exact posterior sampling to update the posterior.
	\item ({\bf Approximate Update}) Use exact posterior sampling to choose an action. Use the approximate method to update the posterior. 
\end{enumerate}
The last two policies are performed to understand how the approximation affects the posterior (discussed in Section~\ref{subsec:discussion}). We update the posterior using the closed-form formula when both the prior and reward distribution are Gaussian in Appendix~\ref{apx:posterior_calculation}. 
\begin{figure}[h]
	\centering
	\begin{subfigure}[ht]{\textwidth}
		\centering
		\includegraphics[width=\textwidth]{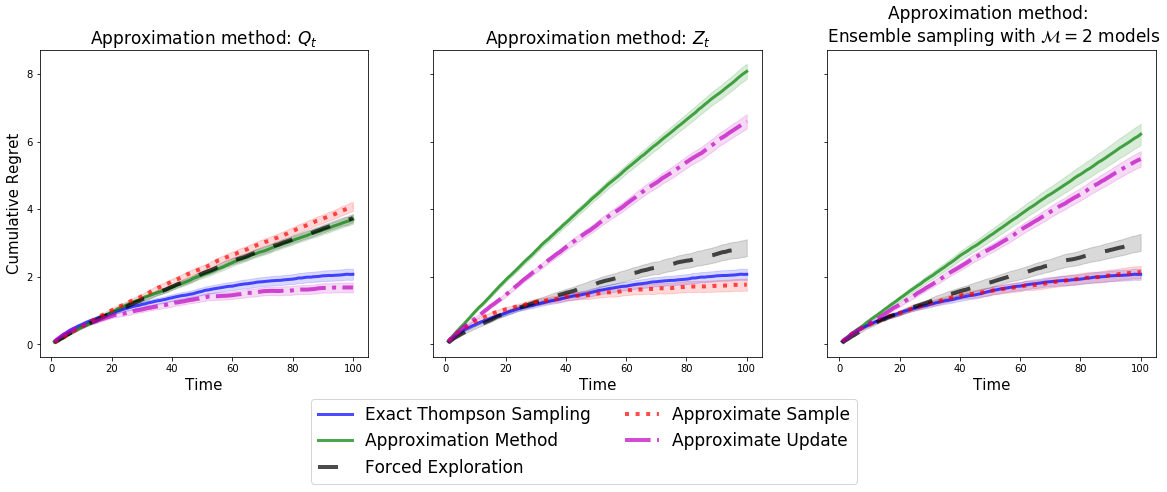}
		\caption{\label{fig:simulation} Applying approximations $Q_t, Z_t$ and Ensemble Sampling to the motivating example (Section~\ref{subsec:small_simulation}). }
	\end{subfigure}
	\begin{subfigure}[ht]{\textwidth}
		\centering
		\includegraphics[width=\textwidth]{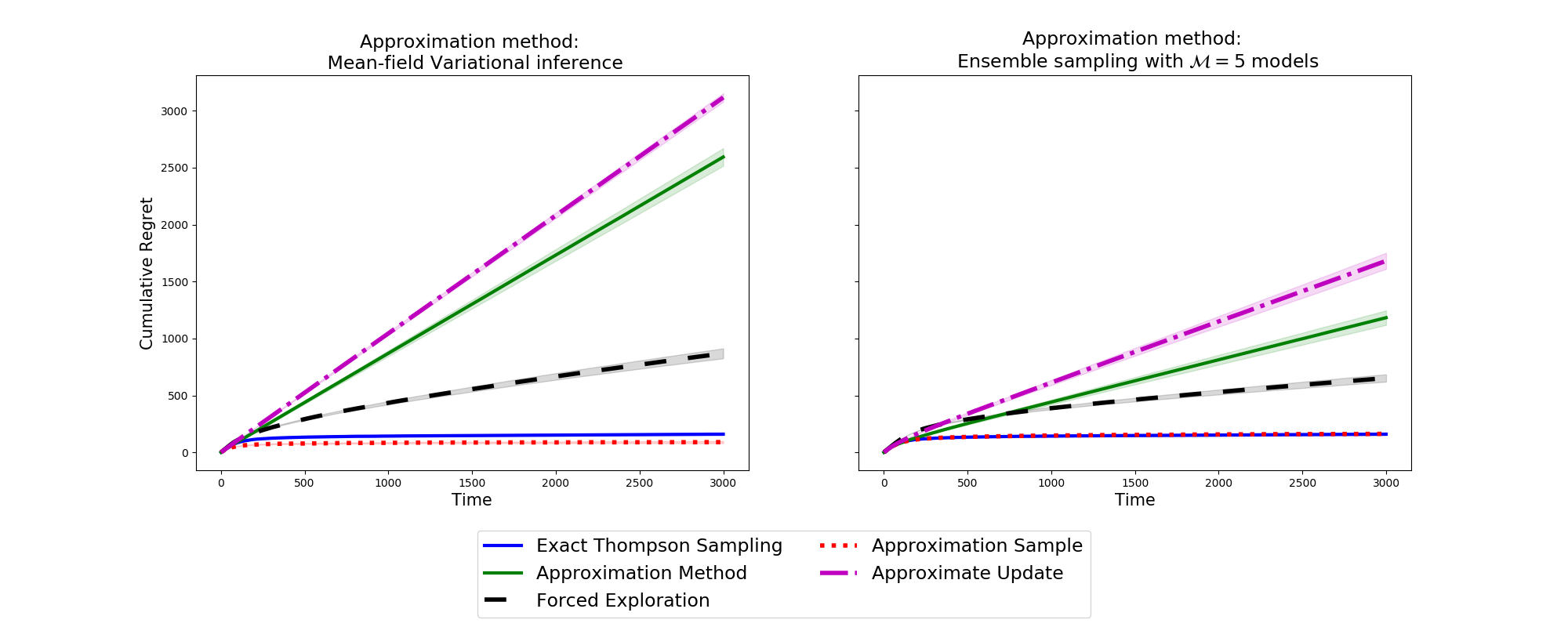}
		\caption{\label{fig:simulation50} Applying mean-field Variational Inference (VI) and Ensemble sampling on a $50$-armed bandit (Section~\ref{subsec:simulation50}). }
	\end{subfigure}
	\caption{Updating the posterior by exact Thompson sampling or adding forced exploration does not help the over-explored approximation $Q_t$, but lowers the regrets of the under-explored approximations $Z_t$, Ensemble sampling and mean-field VI.
		Shaded regions show 95\% confidence intervals.}
\end{figure}
\subsection{Adding Forced Exploration to the Motivating Example}
\label{subsec:small_simulation}

In this section we revisit the example in Section~\ref{sec:motivating_example}. We apply $Q_t, Z_t$ and Ensemble sampling with $\mathcal{M}=2$ models to the bandit problem described in the example. We set the exploration rate at time $t$ to be $1/t$,  $T=100$ and show the results in Figure~\ref{fig:simulation} and discuss them in Section~\ref{subsec:discussion}.

\subsection{Simulations of Ensemble Sampling and Variational Inference for $50$-armed bandits}
\label{subsec:simulation50}

Now we add forced exploration to mean-field Variational Inference (VI) and Ensemble Sampling with $\mathcal{M}=5$ models for a $50$-armed bandit instance. 
We generate the prior and the reward distribution as follows:  the prior is $\mathrm{Norm}(\mathbf{0}, \Sigma_0)$. To generate a positive semi-definite matrix $\Sigma_0$, we generate  a random matrix $A$ of size $(k,k)$ where entries are uniformly sampled from $[0,1)$ and set $\Sigma_0 = A^TA/k$. The true mean $m^*$ is sampled from the prior. The reward distribution of arm $i$ is $\mathrm{Norm}(m^*_i,1)$.

Mean-field VI approximates the posterior by finding an uncorrelated multivariate Gaussian distribution $Q_t$ that minimizes $KL(\Pi_t, Q_t)$.  If the posterior is $\Pi_t = \mathrm{Norm}(\mu_t, \Sigma_t)$ then $Q_t$ has the closed-form solution $Q_t = \mathrm{Norm}(\mu_t, \text{Diag}(\Sigma_t^{-1})^{-1})$, which we used to perform the simulations. We set the exploration rate at time $t$ to be $50/t$, $T=3000$, show the results in Figure~\ref{fig:simulation50} and discuss them in Section~\ref{subsec:discussion}.
 \subsection{Discussion}
 \label{subsec:discussion}
 We observe in Figure~\ref{fig:simulation} that the regret of $Q_t$ calculated from the posterior updated by exact Thompson sampling does not change significantly. Moreover, exact posterior sampling with the posterior updated by $Q_t$ has the same regret as exact Thompson sampling. These two observations imply that $Q_t$ has the same effect on the posterior as exact Thompson sampling.  Therefore adding forced exploration is not helpful. 
 
 On the other hand, in Figures~\ref{fig:simulation} and~\ref{fig:simulation50} the regrets of $Z_t$, Ensemble sampling 
 and mean-field VI calculated from the posterior updated by exact Thompson sampling decrease significantly.  Moreover, exact posterior sampling with the posterior updated by the approximations has similar regret to using the approximations. This behaviour is likely because the approximation causes the posterior to concentrate in the wrong region\footnote{Note that in the case where there are $2$ arms (Figure~\ref{fig:simulation}), exact posterior sampling with the posterior updated by the approximate method has slightly lower regret than naively using the approximate method. This is only because there are only $2$ regions, so exact posterior sampling explores more than the approximation in the other region, which happens to be the correct one.}. In combination, these two observations suggest that these methods do not explore enough for the posterior to concentrate. Therefore adding forced exploration is helpful, which is compatible with the result in Theorem~\ref{thm:uniform_sublinear_regret}. 
 

%% file: related_works.tex
\section{Related Work}
\label{sec:related_work}

There have been many works on sub-linear Bayesian and frequentist regrets for exact Thompson sampling. We discussed relevant works in detail in Section~\ref{sec:alpha_larger_than_0} and Section~\ref{sec:alpha_less_than_1}.

Ensemble sampling \citep{NIPS2017_6918} gives a theoretical analysis of Thompson sampling with one particular approximate inference method. \citet{NIPS2017_6918} maintain a set of $\mathcal{M}$  models to approximate the posterior, and analyzed its regret for linear contextual bandits when $\mathcal{M}$  is $\Omega(\log(T))$. For the $k$-armed bandit problem and when $\mathcal{M}$  is $\Theta(\log(T))$, Ensemble sampling satisfies the condition $\mathrm{KL}(\overline{Q}_t, \overline{\Pi}_t) < \epsilon$ in Theorem~\ref{thm:uniform_sublinear_regret} with high probability. In this case, the regret of Ensemble sampling scales linearly with $T$. 

We show in  Theorem~\ref{thm:alpha_less_than_1_example} that when the constraint $\mathrm{KL}(Q_t, \Pi_t) < \epsilon$ is satisfied, which implies by Lemma~\ref{lem:convert_distribution} that $\mathrm{KL}(\overline{Q}_t, \overline{\Pi}_t) < \epsilon$ is satisfied, there can exist approximation algorithms that have linear regret in $T$. This result provides a linear lower bound, which is complementary with the linear regret upper bound of Ensemble Sampling in \citep{NIPS2017_6918}. Moreover, we show in Appendix~\ref{sec:related_works_appendix} that we can apply  Theorem~\ref{thm:uniform_sublinear_regret} to get sub-linear regret with Ensemble sampling with $\Theta(\log(T))$ models. 

In reinforcement learning, there is a notion that certain approximations are "stochastically optimistic" and that this has implications for regret \citep{Osband:2016:GEV:3045390.3045641}. This is similar in spirit to our analysis in terms of $\alpha$-divergence, in that the characteristics of inference errors are important.

There has been a number of empirical works using approximate methods to perform Thompson sampling. \citet{DBLP:journals/corr/abs-1802-09127} implement variational inference, MCMC, Gaussian processes and other methods on synthetic and real world data sets and measure the regret. \citet{pmlr-v84-urteaga18a} derive a variational method for contextual bandits. \citet{NIPS2015_5985} use particle filtering to implement Thompson sampling for matrix factorization.

Finally, if exact inference is not possible, it remains an open question if it is better to use Thompson sampling with approximate inference, or to use a different bandit method that does not require inference with respect to the posterior. For example \citet{pmlr-v97-kveton19a} propose an algorithm based on the bootstrap.

%% file: conclusion.tex
\section{Conclusion}

In this paper we analyzed the performance of approximate Thompson sampling when at each time-step $t$, the algorithm obtains a sample from an approximate distribution $Q_t$ such that the $\alpha$-divergence between the true posterior and $Q_t$ remains at most a constant $\epsilon$ at all time-steps. 

Our results have the following implications. To achieve a sub-linear regret, we can only use $\alpha >0$ for $o(T)$ time-steps. Therefore we should use $\alpha \leq 0$ with forced exploration to make the posterior concentrate. This method theoretically guarantees a sub-linear regret even when $\epsilon$ is a large constant.


%% file: appendix.tex
\appendix
\input{proof_thm_alpha_larger_than_0_copy.tex}
\input{proof_thm_alpha_smaller_than_1_copy.tex}
\input{proof_convert_distribution_copy.tex}
\input{proof_uniform_copy.tex}
\input{related_works_appendix.tex}
\section{KL Divergence between two Gaussian Distributions} 
\label{sec:KL_divergence_gaussian}
The KL divergence between two Gaussian distributions is:
\begin{align*}
    &\mathrm{KL}(\mathrm{Norm}(\mu_1, \Sigma_1), \mathrm{Norm}(\mu_2, \Sigma_2)) \\ &= \frac{1}{2}  ( \mathrm{trace}(\Sigma_2^{-1}\Sigma_1)  - k\\ & + (\mu_2 - \mu_1)^T \Sigma_2^{-1}(\mu_2 - \mu_1) +  \ln\frac{\mathrm{det}\Sigma_2}{\mathrm{det}\Sigma_1} ) 
\end{align*}
\section{Posterior Calculation}
\label{apx:posterior_calculation}
In our simulations, when both the prior and the reward distributions are Gaussian, we calculate the true posterior using the following closed-form solution. 

Let the posterior at time $t$ be multivariate Gaussian distribution $\mathrm{Norm}(\mu_t, \Sigma_t)$ where $\mu_t$ is a $k \times 1$ vector and $\Sigma_t$ is a $k \times k$ covariance matrix. Let the reward distribution of arm $i$ be $\mathrm{Norm}(m^*_i, \sigma^2)$ where $\sigma$ is known and $m^*_i$'s are unknown. 

Let $A_t \in \{0,1\}^k$ be a $0/1$ vector where $A_t(i) =1$ if arm $i$ is chosen at time $t$, and $0$ otherwise. Let $r_t \in \mathcal{R}$ be the reward of the arm chosen at time $t$.

Then the posterior at time $t+1$ is $\mathrm{Norm}(\mu_{t+1}, \Sigma_{t+1})$ where:
\begin{align*}
\Sigma_{t+1} = (\Sigma_t^{-1} + A_tA_t^T/\sigma^2)^{-1} \\
\mu_{t+1} = \Sigma_{t+1} (\Sigma_t^{-1} \mu_t+A_t r_t /\sigma^2 )\;.
\end{align*}

%% file: proof_thm_alpha_larger_than_0_copy.tex
\section{Proof of Theorem \ref{thm:alpha_larger_than_0} and Corollary \ref{col:alpha_larger_than_0}}
\label{sec:proof_alpha_larger_than_0}
First we will prove Theorem~\ref{thm:alpha_larger_than_0}.  Let $\Omega_i \subseteq \Omega$ denote the region where arm $i$ is the best arm. Let $\Pi_{t,i}$ denote $\Pi_t(\Omega_i)$, the posterior probability that arm $i$ is the best arm. For $r > 1$, We construct the pdf of $Q_t$'s as follows: 
\begin{align}
\label{eq:Q_t_construction_alpha_larger_than_0}
q_t(m) = 
\begin{cases}
\frac{1}{r} \pi_t(m),  &\text{if } m_1 > m_2 \\
\frac{1 - \Pi_{t,1}/r}{1 - \Pi_{t,1}} \pi_t(m), &\text{ otherwise.}
\end{cases}
\end{align}

We will prove the theorem by the following steps:
\begin{itemize}
    \item In Lemma~\ref{lem:valid_distribution_alpha_larger_than_0} we show that $Q_t$'s are valid distributions. 
    \item In Lemma~\ref{lem:small_divergence1} we show that when $\alpha > 0$ the $\alpha$-divergence between $Q_t$ and $\Pi_t$ can be arbitrarily small 
    \item In Lemma~\ref{lem:regret1} we show that sampling from $Q_t$ for $\Theta(T)$ time-steps will generate linear frequentist regret, and lower bound the regret. 
\end{itemize}
Since the regret is linear, in Appendix~\ref{apx:average_regret_alpha_larger_than_0} we discuss the constant average regret per time-step as a function of $\epsilon$ and $\alpha$. 
In Appendix~\ref{apx:alpha_larger_than_0_bayesian} we provide the Bayesian regret proof for Corollary ~\ref{col:alpha_larger_than_0}.

\begin{lemma}
\label{lem:valid_distribution_alpha_larger_than_0}
$q_t(m)$ in Eq.~\ref{eq:Q_t_construction_alpha_larger_than_0} is well-defined and if $\int \pi_t(m) dm = 1$ then:
\begin{align*}
\int q_t(m) dm = 1. 
\end{align*}
\end{lemma}

\begin{lemma}
\label{lem:small_divergence1}
When $\alpha >0$, for all $\epsilon > 0$, for all $\Pi_t$,
there exists $r >1$  such that  when $Q_t$'s are constructed from $r$ as shown in  Eq.~\ref{eq:Q_t_construction_alpha_larger_than_0}, $D_{\alpha}(\Pi_t, Q_t) < \epsilon$
\end{lemma}
\begin{lemma}
\label{lem:regret1}
The expected frequentist regret of the policy that constructs $Q_t$'s as in Eq.~\ref{eq:Q_t_construction_alpha_larger_than_0} and sample from $Q_t$ for $T' =\Theta(T)$ time-steps is linear and the lower bound of the average regret per time-step is
   \begin{align*}
    L =
    \begin{cases}
    c\Delta(1- (1-\epsilon \alpha(1-\alpha))^{\frac{1}{1-\alpha}}), &\text{ when } \alpha > 1 \text{ and } 0<\epsilon \\
    c\Delta(1 - \frac{1}{e^{\epsilon}}), & \text{ when } \alpha = 1 \text{ and } 0<\epsilon\\
     c\Delta(1- (1-\epsilon  \alpha(1-\alpha))^{\frac{1}{1-\alpha}}),    & \text{ when } 0< \alpha < 1\text{ and } 0 < \epsilon \leq \frac{1}{\alpha(1-\alpha)}\;. 
     \end{cases}, 
\end{align*}
where $c = \frac{T'}{T}$ is $\Theta(1)$.
\end{lemma}
\subsection{Proof of Lemma~\ref{lem:valid_distribution_alpha_larger_than_0}}
\begin{proof}
First we will show that $\Pi_{t,2}=1-\Pi_{t,1}>0$ for all $t \geq 0$, so that $q_t(m)$ is well-defined. We have $\Pi_{0,2} = \mathbb{P}(M_2>M_1)>0$ by assumption. Let $S_t= \{m \in \Omega_2: \pi_t(m) >0\}$ be the support of $\Pi_t$ in $\Omega_2$. If $\pi_0(m) >0$, then $\pi_t(m)>0$ because $\pi_t(m)$ is the product of $\pi_0(m)$ and non-zero likelihoods. Therefore $S_0 \subseteq S_t$. 

Since $\mathbb{P}(M_2 > M_1) = \int_{S_0} \pi_0(m) dm>0$, $\int_{S_0} dm >0$. Since $S_0 \subseteq S_t$, $\int_{S_t} dm >0$. Therefore $\int_{S_t} \pi_t(m) dm>0$ since $S_t= \{m \in \Omega_2: \pi_t(m) >0\}$ by definition. Then $\Pi_{t,2} = \int_{\Omega_2} \pi_t(m) dm = \int_{S_t} \pi_t(m) dm >0$.

Assume that $\int \pi_t(m) dm = 1$, we will show that $\int q_t(m) dm = 1$: 
\begin{align*}
  &\int q_t(m) dm \\ 
  &= \int_{\Omega_1}  q_t(m) dm + \int_{\Omega_2}  q_t(m) dm \\ 
  &= \int_{\Omega_1} \frac{1}{r} \pi_t(m) dm + \int_{\Omega_2}   \frac{1 - \Pi_{t,1}/r}{1 - \Pi_{t,1}} \pi_t(m) dm \\
  &= \frac{1}{r} \Pi_{t,1} + \frac{1 - \Pi_{t,1}/r}{1 - \Pi_{t,1}} \Pi_{t,2} \\
  &= \frac{1}{r} \Pi_{t,1} + \frac{1 - \Pi_{t,1}/r}{1 - \Pi_{t,1}} (1-\Pi_{t,1}) \\
  &= 1\;.
\end{align*}
\end{proof}
\subsection{Proof of Lemma~\ref{lem:small_divergence1}}
\begin{proof}
First we calculate the $\alpha$-divergence between $\Pi_t$ and $Q_t$ constructed in Eq.~\ref{eq:Q_t_construction_alpha_larger_than_0}. Let $\Omega_1 \subseteq \Omega$ denote the region where $m_1 > m_2$ and $\Omega_2 \subseteq \Omega$ denote the region where $m_2 \geq m_1$.

When $\alpha > 0, \alpha \neq 1$ we have:
\begin{align}
&D_{\alpha} (\Pi_t, Q_t) \nonumber \\
=&  \frac{1- \int \left ( \frac{\pi_t(m)}{q_t(m)}\right ) ^{\alpha}q_t(m) dm}{\alpha(1-\alpha)} \nonumber\\ 
=&  \frac{1- \int_{\Omega_1} \left ( \frac{\pi_t(m)}{q_t(m)}\right ) ^{\alpha} q_t(m) dm -\int_{\Omega_2} \left (\frac{\pi_t(m)}{q_t(m)}\right ) ^{\alpha}q_t(m) dm }{\alpha(1-\alpha)} \nonumber\\ 
=&  \frac{1- \int_{\Omega_1} \left ( r \right ) ^{\alpha}q_t(m) dm - \int_{\Omega_2} \left ( \frac{1 - \Pi_{t,1}}{1 - \Pi_{t,1}/r} \right ) ^{\alpha}q_t(m) dm}{\alpha(1-\alpha)} \nonumber\\ 
=&  \frac{1- Q_t(\Omega_1)\left ( r \right ) ^{\alpha} - Q_t(\Omega_2) \left ( \frac{1 - \Pi_{t,1}}{1 - \Pi_{t,1}/r} \right ) ^{\alpha}}{\alpha(1-\alpha)} \nonumber\\ 
=&  \frac{1- \frac{\Pi_{t,1}}{r}\left ( r \right ) ^{\alpha} - (1 - \frac{\Pi_{t,1}}{r}) \left ( \frac{1 - \Pi_{t,1}}{1 - \Pi_{t,1}/r} \right ) ^{\alpha}}{\alpha(1-\alpha)} \nonumber\\ 
\label{eq:max_d_alpha_large}
=& \frac{1}{\alpha(1-\alpha)} \left ( 1 -  \Pi_{t,1} r^{-1+\alpha}- (1-\Pi_{t,1})^{\alpha} (1-\frac{\Pi_{t,1}}{r})^{1-\alpha} \right)\;.
\end{align}
When $\alpha = 1$:
\begin{align*}
   & D_{\alpha} (\Pi_t, Q_t) \\
=&  \int \pi_t(m) \log \left ( \frac{\pi_t(m)}{q_t(m)} \right ) dm \\ 
=& \int_{\Omega_1} \pi_t(m) \log \frac{\pi_t(m)}{q_t(m)} dm + \int_{\Omega_2} \pi_t(m) \log  \frac{\pi_t(m)}{q_t(m)} \ dm \\
=& \int_{\Omega_1} \pi_t(m) \log(r) dm \\&+ \int_{\Omega_2} \pi_t(m) \log \frac{1 - \Pi_{t,1}}{1 - \Pi_{t,1}/r} dm \\
=& \Pi_{t,1}\log(r) + (1-\Pi_{t,1})\log \frac{1 - \Pi_{t,1}}{1 - \Pi_{t,1}/r}\;.
\end{align*}
We will now upper bound the above expression. Consider 2 cases
\begin{itemize}
    \item $\alpha =1$:   We have
\begin{align*}
   & D_{\alpha} (\Pi_t, Q_t) \\
&= \Pi_{t,1}\log(r) + (1-\Pi_{t,1})\log \frac{1 - \Pi_{t,1}}{1 - \Pi_{t,1}/r} \\
&\leq \Pi_{t,1}\log(r) + (1-\Pi_{t,1})\log(r) \text{ because } r>1\\
&\leq \log(r)\;.
\end{align*}

    \item $\alpha > 0, \alpha \neq 1$: 
The following inequality is true by simple calculations when $0< \alpha < 1$ or $\alpha > 1$: 
    \begin{align}
    \label{eq:upper_bound_error}
    \frac{\left (\frac{ 1-\Pi_{t,1}}{1-\frac{\Pi_{t,1}}{r}} \right ) ^{\alpha-1} }{\alpha(\alpha-1)} \leq \frac{ r^{\alpha-1}}{\alpha(\alpha -1 )}\;. 
\end{align}

Then we will have: 
    \begin{align*}
  &D_{\alpha} (\Pi_t, Q_t)  \\
  = &\frac{\Pi_{t,1} r^{\alpha-1} + (1-\Pi_{t,1})\left (\frac{ 1-\Pi_{t,1}}{1-\frac{\Pi_{t,1}}{r}} \right ) ^{\alpha-1} -1}{\alpha(\alpha-1 )}  \\
 \leq & \frac{1}{\alpha(\alpha-1 )} \left (  \Pi_{t,1} r^{\alpha-1} + (1-\Pi_{t,1}) r^{\alpha-1} -1 \right) \\
 = &  \frac{1}{\alpha(\alpha-1 )} \left (   r^{-1+\alpha}  -1 \right)\;.
\end{align*}

%
%
\end{itemize}
Therefore $D_{\alpha} (\Pi_t, Q_t)$ is upper bounded by:
\begin{align}
\label{eq:large_alpha}
    \begin{cases}
    \frac{1- {r}^{\alpha-1}}{\alpha(1-\alpha)}, &\text{ if }  0 < \alpha <1  \text{ or } \alpha > 1\\
    \log(r), &\text{ if } \alpha = 1\;.
    \end{cases}
\end{align}
Since $\lim_{r \rightarrow 1} \log(r) = 0$ and $ \lim_{r \rightarrow 1} \frac{1- r^{-1+\alpha}}{\alpha(1-\alpha)} =0$, for any $\epsilon > 0$, there exists $r > 1$ such that
\begin{align*}
D_{\alpha} (\Pi_t, Q_t) \leq \epsilon\;.
\end{align*}
\end{proof}
\subsection{Proof of Lemma~\ref{lem:regret1}}
\begin{proof}
We will now lower bound the regret as a function of $\epsilon$. 

For any posterior $\Pi_t$, since the approximate algorithm sampling from $Q_t$ picks the optimal arm with probability at most $1/r$ it then picks arm $2$ with probability at least $1-1/r$. 


Since we sample from $Q_t$ for $T'$ time steps, the lower bound of the average expected regret per time step is :
\begin{align*}
    L =\frac{T'}{T} (m^*_1 - m^*_2)(1- 1/r) = c \Delta(1-1/r)\;.
\end{align*}
where $\Delta = m^*_1 - m^*_2$ and $c=\frac{T'}{T}$ is $\Theta(1)$.

We calculate $\epsilon$ as a function of $r$ from Eq.~\ref{eq:large_alpha}:
\begin{align*}
    \epsilon &=  
    \begin{cases}
    \frac{1- r^{-1+\alpha}}{\alpha(1-\alpha)}, &\text{ if } \alpha \neq 1 \\
    log(r), &\text{ if } \alpha = 1\;.
    \end{cases}
\end{align*}
 The functions are continous when $r>1$. Then by direct calculations when $r\rightarrow \infty$ and $r\rightarrow 1$, the domain of $\epsilon$ is:
 \begin{align*}
0 &< \epsilon \text{ when } \alpha \geq 1\;. \\
0 &< \epsilon < \frac{1}{\alpha(1-\alpha)} \text{ when } 0<\alpha <1\;. 
 \end{align*}
Then
\begin{align*}
    r =
    \begin{cases}
    (1 - \epsilon \alpha (1-\alpha))^{\frac{1}{-1+\alpha}} &\text{ when } \alpha > 1 \text{ and } 0<\epsilon \\
   e^{\epsilon}&\text{ when } \alpha = 1 \text{ and } 0<\epsilon \\
    (1 - \epsilon \alpha (1-\alpha))^{\frac{1}{-1+\alpha}} & \text{ when } 0< \alpha < 1 \text{ and } 0 < \epsilon \leq \frac{1}{\alpha(1-\alpha)}. 
    \end{cases}
\end{align*}
Therefore we can calculate the lower bound of the regret per time-step as:
   \begin{align*}
   L =
   \begin{cases}
   c\Delta(1- (1-\epsilon \alpha(1-\alpha))^{\frac{1}{1-\alpha}}), &\text{ when } \alpha > 1 \text{ and } 0<\epsilon \\
   c\Delta(1 - \frac{1}{e^{\epsilon}}), & \text{ when } \alpha = 1 \text{ and } 0<\epsilon\\
   c\Delta(1- (1-\epsilon  \alpha(1-\alpha))^{\frac{1}{1-\alpha}}),    & \text{ when } 0< \alpha < 1 \text{ and } 0 < \epsilon \leq \frac{1}{\alpha(1-\alpha)}. 
   \end{cases}.
   \end{align*}
We plot the lower bound of the average regret per time step when $\Delta = 0.1$ as a function of $\epsilon$ in Fig~\ref{fig:regret_epsilon1}. 

\end{proof}
\subsection{The Average Regret per Time-step}
\label{apx:average_regret_alpha_larger_than_0} 
To understand how the constant average regret per time-step depends on $\epsilon$ and $\alpha$, we plot in Figure~\ref{fig:regret_epsilon1} the lower bound of the average regret per time-step in Lemma~\ref{lem:regret1} as a function of $\epsilon$ in the following setting of the example constructed in the proof of Theorem~\ref{thm:alpha_larger_than_0}. The algorithm samples from $Q_t$ at $T/2$ time-steps and $\Delta = 0.1$. In this case the average regret per time step is upper bounded by $\Delta/2 = 0.05$. The formula and proof are detailed in Lemma~\ref{lem:regret1} in Appendix~\ref{sec:proof_alpha_larger_than_0}. 
\begin{figure}[ht]
	\centering
	\includegraphics[width=\columnwidth, height = 5cm, keepaspectratio]{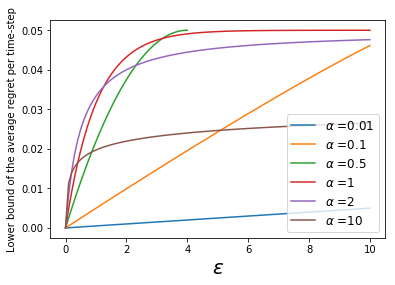}
	\caption{\label{fig:regret_epsilon1} Lower bound of regret per time-step as a function of $\epsilon$ when $m^*_1 - m^*_2 = 0.1$ and we sample from the approximation for $T/2$ time-steps in the example construted in the proof of Theorem~\ref{thm:alpha_larger_than_0}. When $\alpha$ is around $1$, the lower bound converges quickly as $\epsilon$ increases.}
\end{figure}
When $\alpha$ is around $1$, the lower bound, and therefore the average regret per time-step, converges the fastest to $\Delta/2$ as $\epsilon$ increases. When $\alpha$ is very large or close to $0$, the lower bound grows slowly as $\epsilon$ increases.
\subsection{Proof of Corollary~\ref{col:alpha_larger_than_0}}
\label{apx:alpha_larger_than_0_bayesian}
Since $\mathbb{P}(M_1 > M_2 ) >0$, there exist constants $\Delta>0, \gamma > 0$ such that $\mathbb{P}(M_1 - M_2 \geq \Delta) = \gamma$. The probability that the assumption $m^*_1>m^*_2$ in Theorem~\ref{thm:alpha_larger_than_0} is satisfied is at least $\gamma>0$. Therefore the expected regret over the prior is at least $\gamma$ times the frequentist regret in Theorem~\ref{thm:alpha_larger_than_0}, which is linear. 

%% file: proof_thm_alpha_smaller_than_1_copy.tex
\section{Proof of Theorem \ref{thm:alpha_less_than_1_example} and Corollary \ref{col:alpha_less_than_1_example}}
\label{sec:proof_alpha_less_than_1}
First we will prove Theorem~\ref{thm:alpha_less_than_1_example}. Let $\Pi_{t,i}$ denote $\Pi_t(\Omega_i)$. We construct the pdf of $Q_t$'s as follows: 
\begin{align}
\label{eq:Q_t_construction_alpha_less_than_1}
q_t(m) = 
\begin{cases}
\frac{1}{\Pi_{t,2}} \pi_t(m), \text{ if } m_2 > m_1 \\
0, \text{ otherwise}. 
\end{cases}
\end{align}

We will prove the theorem by the following steps:
\begin{itemize}
    \item In Lemma~\ref{lem:valid_distribution_alpha_smaller_than_1} we show that $Q_t$'s are valid distributions. 
    \item In Lemma~\ref{lem:regret_alpha_less_than_1} we show that $Q_t$ has linear frequentist regret, and calculate the constant average regret per time-step.
    \item In Lemma~\ref{lem:small_divergence_alpha_less_than_1} we show that there exists a bad prior such that the $\alpha$-divergence between $Q_t$ and $\Pi_t$ can be arbitrarily small.
\end{itemize}
In Appendix~\ref{apx:prior_alpha_less_than_1} we discuss the prior-dependent error threshold $\epsilon$ that will cause linear regret. 
In Appendix~\ref{apx:alpha_less_than_1_bayesian} we provide the Bayesian regret proof for Corollary~\ref{col:alpha_less_than_1_example}. 
\begin{lemma}
\label{lem:valid_distribution_alpha_smaller_than_1}
 $q_t(m)$  in Eq.~\ref{eq:Q_t_construction_alpha_less_than_1} is well-defined and if $\int \pi_t(m) dm = 1$ then:
\begin{align*}
\int q_t(m) dm = 1. 
\end{align*}
\end{lemma}

\begin{lemma}
	\label{lem:regret_alpha_less_than_1}
	$Q_t$ constructed in Eq.~\ref{eq:Q_t_construction_alpha_less_than_1} chooses arm $2$ at all time-steps. The average frequentist regret per time-step is $\Delta = m^*_1-m^*_2$. 
\end{lemma}

\begin{lemma}
\label{lem:small_divergence_alpha_less_than_1}
Let $\alpha < 1$,  $M_1-M_2$ and $M_2$ be independent and arm $2$ be chosen at all time-steps before $t$. 

For any $ \epsilon > 0$, there exists $0<z\leq 1$ such that if $\Pi_{0,2}=z$ then $D_{\alpha} (\Pi_t, Q_t) < \epsilon$ where $Q_t$ is constructed in Eq.~\ref{eq:Q_t_construction_alpha_less_than_1}.

For any $0<z\leq 1$, there exists $\epsilon >0 $ such that if $\Pi_{0,2}=z$  then $D_{\alpha} (\Pi_t, Q_t) < \epsilon$ where $Q_t$ is constructed in Eq.~\ref{eq:Q_t_construction_alpha_less_than_1}.
\end{lemma}

\subsection{Proof of Lemma~\ref{lem:valid_distribution_alpha_smaller_than_1}}
\begin{proof}
Similar to the proof of Lemma~\ref{lem:valid_distribution_alpha_larger_than_0}, we have that $\Pi_{t,2} >0$ for all $t\geq 0$. 

Assume that $\int \pi_t(m) dm = 1$, we will show that $\int q_t(m) dm = 1$: 
\begin{align*}
  &\int q_t(m) dm \\ 
  &= \int_{\Omega_1}  q_t(m) dm + \int_{\Omega_2}  q_t(m) dm \\ 
  &= 0 + \int_{\Omega_2}   \frac{1}{ \Pi_{t,2}} \pi_t(m) dm \\
  &= \frac{1}{ \Pi_{t,2}} \int_{\Omega_2}    \pi_t(m) dm \\
  &= 1\;.
\end{align*}
\end{proof}

\subsection{Proof of Lemma~\ref{lem:regret_alpha_less_than_1}}
\begin{proof}
	Under the approximate distribution, arm $2$ is chosen with probability $1$ at all times. 
	Clearly this approximate distribution has linear regret, with $\Delta=m^*_1 - m^*_2$ being the average regret per time-step.
\end{proof}

    \subsection{Proof of Lemma~\ref{lem:small_divergence_alpha_less_than_1}}
\begin{proof} 
Let $D = M_1 - M_2$ which is independent of $M_2$ by the assumption. Let  $f$ denote the pdf. Since the algorithm always picks arm $2$, $H_{t-1}$ and $M_1$ are independent given $M_2$. Therefore for all $m_1, m_2$ and $h$, $f_{M_1|M_2,H_{t-1}}(m_1 |m_2, h) = f_{M_1 |M_2}(m_1|m_2) $. 

Since $D = M_1-M_2$, we have $f_{D|M_2,H_{t-1}}(m_1-m_2 |m_2, h)  = f_{M_1|M_2,H_{t-1}}(m_1 |m_2, h) $. Therefore for all $d, m_2$ and $h$:
\begin{align*}
f_{D|M_2,H_{t-1}}(d |m_2, h)  = f_{M_1|M_2,H_{t-1}}(m_2+d |m_2, h) = f_{M_1 |M_2}(m_2+d|m_2) =f_{D |M_2}(d|m_2)\;.
\end{align*}                                       
Since $f_{D|M_2,H_{t-1}}(d |m_2, h)  =f_{D |M_2}(d|m_2)$ for all $d, m_2$ and $h$,  $D$ and $H_{t-1}$ are independent given $M_2$. Then
\begin{align*}
&f_{D |M_2, H_{t-1}}(d|m_2, h)  \\ =  &f_{D |M_2}(d|m_2) \text{ because $D$ and $H_{t-1}$ are independent given $M_2$} \\ = &f_{D}(d) \text{ because $D$ and $M_2$ are independent.} 
\end{align*}
Now we will show that $D$ and $H_{t-1}$ are independent. For all $d$ and $h$: 
\begin{align*}
&f_{D | H_{t-1}}(d|h) \\
=&\int f_{D, M_2| H_{t-1}} (d, m_2|h) dm_2\\
= &\int f_{D| M_2, H_{t-1}} (d| m_2,h)  f_{ M_2| H_{t-1}} ( m_2|h) dm_2 \\
=  &\int f_{D}(d)  f_{ M_2| H_{t-1}} ( m_2|h) dm_2  \\
=  &f_{D}(d) \int f_{ M_2| H_{t-1}} ( m_2|h) dm_2  \\
=  & f_{D}(d)\;. 
\end{align*}
Since $D$ and $H_{t-1}$ are independent, at all times $t$ the posterior does not concentrate: 
\begin{align*}
\Pi_{t,2}& =\mathbb{P}(M_1-M_2 < 0 |H_{t-1}) = \mathbb{P}(M_1< M_2)\;.
\end{align*}
For simplicity let
\begin{align*}
	z:=\mathbb{P}(M_1< M_2)\;.
\end{align*}
We will show that  $D(\Pi_t, Q_t)$ is small if $z$ is large enough. First we calculate the $\alpha$-divergence between $\Pi_t$ and $Q_t$ constructed in Eq~\ref{eq:Q_t_construction_alpha_less_than_1}. 

When $\alpha <1, \alpha \neq 0$: 
\begin{align*}
&D_{\alpha} (\Pi_t, Q_t)  \\
=&  \frac{1- \int \left ( \frac{q_t(m)}{\pi_t(m)}\right ) ^{1-\alpha} \pi_t(m) dm}{\alpha(1-\alpha)} \\ 
=&  \frac{1 - \int_{\Omega_1} \left ( \frac{q_t(m)}{\pi_t(m)}\right ) ^{1-\alpha} \pi_t(m) dm -\int_{\Omega_2} \left ( \frac{q_t(m)}{\pi_t(m)}\right ) ^{1-\alpha} \pi_t(m) dm }{\alpha(1-\alpha)} \\ 
=&  \frac{1- 0  - \int_{\Omega_2} \left ( \frac{1}{\Pi_{t,2}}\right ) ^{1-\alpha} \pi_t(m) dm}{\alpha(1-\alpha)} \text{ since } \alpha < 1 \\
=&  \frac{1- \left (\frac{1}{\Pi_{t,2}} \right ) ^{1-\alpha}\int_{\Omega_2}   \pi_t(m) dm}{\alpha(1-\alpha)} \\
=&  \frac{1- \left (\frac{1}{\Pi_{t,2}} \right ) ^{1-\alpha} \Pi_{t,2}}{\alpha(1-\alpha)} \\
=&  \frac{1 - (\Pi_{t,2}) ^{\alpha}}{\alpha(1-\alpha)} \\ 
=& \frac{1-z^{\alpha}}{\alpha(1-\alpha)}\;.
\end{align*}
When $\alpha = 0$: 
\begin{align*}
    &D_{\alpha} (\Pi_t, Q_t) \\ =& \int q_t(m) \log  \frac{q_t(m)}{\pi_t(m)} dm \\
    =& \int_{\Omega_1} q_t(m) \log  \frac{q_t(m)}{\pi_t(m)} dm  \\ &+ \int_{\Omega_2} q_t(m) \log \frac{q_t(m)}{\pi_t(m)} dm  \\
    =& \int_{\Omega_1} 0 \log(0)dm + \int_{\Omega_2} q_t(m) \log \frac{1}{\Pi_{t,2}}dm \\
    = &0 + 1 \log  \frac{1}{\Pi_{t,2}} = \log  \frac{1}{\Pi_{t,2}}= \log  \frac{1}{z}\;.
\end{align*}
Note that if we don't have the condition on the prior such that picking arm $2$ does not help to learn which arm is the better one, $\Pi_{t,2}$ may converge to $0$, making $D_{\alpha} (\Pi_t, Q_t)$ goes to $\infty$ when $\alpha \leq 0 $. But since $\Pi_{t,2}=z$, we will now show that for any $\alpha <1$, for any $\epsilon >0$, there exists $z (0 <z <1)$ such that  
\begin{align*}
    D_{\alpha} (\Pi_t, Q_t) < \epsilon\;.
\end{align*}
Consider the 2 cases
\begin{itemize}
    \item When $\alpha < 1, \alpha \neq 0$: Since
    \begin{align*}
   \lim_{z \rightarrow 1}  \frac{1-z^{\alpha}}{\alpha(1-\alpha)} =0\;.
   \end{align*}
   Then for any $\epsilon > 0 $ there exists $ 0 < z<1$ such that $D_{\alpha}(\Pi_t, Q_t) < \epsilon$.
   For any $0<z< 1 $ there exists $\epsilon > 0$ such that $D_{\alpha}(\Pi_t, Q_t) < \epsilon$. 
   
   \item  When $\alpha = 0$: 
    \begin{align*}
    D_{\alpha} (\Pi_t, Q_t) &=   \log  \frac{1}{z}\;.
\end{align*}
Since $\lim_{z \rightarrow 1} \log(1/z) = 0 $, for any $\epsilon > 0 $ there exists $0 < z< 1$ such that $D_0(\Pi_t, Q_t) < \epsilon$. For any $z< 1 $ there exists $\epsilon > 0$ such that $D_{\alpha}(\Pi_t, Q_t) < \epsilon$. 
\end{itemize}
\end{proof}
\subsection{Prior-dependent Error Threshold for Linear Frequentist Regret}
\label{apx:prior_alpha_less_than_1} 
In the example constructed in the previous sections, the $\alpha$-divergence between $\Pi_t$ and $Q_t$ can be calculated as:
$   \epsilon = 
\begin{cases}
\frac{1- z^\alpha}{\alpha(1-\alpha)}, \text{ if } 0 < \alpha < 1 \text{ or } \alpha <0 \\
\log\frac{1}{z}, \text{ if } \alpha = 0 \\
\end{cases}\;.$

\begin{figure}[h]
	\centering
	\begin{subfigure}[t]{0.5\columnwidth}
		\raggedleft
		\captionsetup{width=.9\linewidth}
		\includegraphics[width=1\columnwidth]{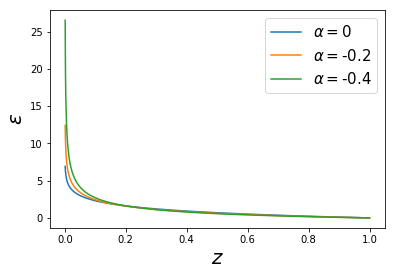}
		\caption{$D_{\alpha} (\Pi_t, Q_t) =\epsilon $ as a function of $z$ when $\alpha \leq 0$. When $z$ is very small and $\alpha$ is small, $\epsilon$ needs to be very large. When $z>0.2$, there is a threshold of $\epsilon$ which is less than $8$ that can cause linear regret. }
	\end{subfigure}%
	~ 
	\begin{subfigure}[t]{0.5\columnwidth}
		\raggedright 
		\captionsetup{width=.9\linewidth}
		\includegraphics[width=1\columnwidth]{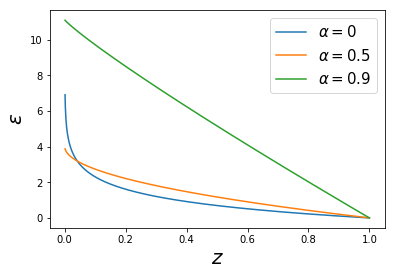}
		\caption{$\epsilon$ as a function of $z$ when $0\leq \alpha < 1$. There is a threshold of $\epsilon$ which is less than $8$ for each value of $z$ that can cause linear regret..}
	\end{subfigure}
	\caption{\label{fig:small_alpha}$\epsilon$ as a function $z$ that makes the regret linear for different values of $\alpha$ for the example constructed in the proof of Theorem~\ref{thm:alpha_less_than_1_example}.}
\end{figure}
In Figure~\ref{fig:small_alpha}, we show the values of $\epsilon$ as a function of $z$ that will make the regret linear for different values of $\alpha$. We can see that for both cases when $\alpha \leq 0$ and $0 \leq \alpha <1$, and $z$ is not too small, there is a threshold of $\epsilon$ for each value of $z$ that makes the regret linear. For each value of $z$, if the error is smaller than the threshold we hypothesize that the regret might become sub-linear. However even if that is the case, it is not possible to calculate the exact threshold for more complicated priors. Therefore in Section~\ref{subsec:uniform} we propose an algorithm that is guaranteed to have sub-linear regret for any $\epsilon$ and any $z$ when $\alpha \leq 0$.

\subsection{Proof of Corollary~\ref{col:alpha_less_than_1_example}}
\label{apx:alpha_less_than_1_bayesian}
Since $\mathbb{P}(M_1 > M_2 ) >0$, there exist constants $\Delta>0, \gamma > 0$ such that the $\mathbb{P}(M_1 - M_2 \geq \Delta) = \gamma$. The probability that the assumption $m^*_1>m^*_2$ in Theorem~\ref{thm:alpha_less_than_1_example} is satisfied is at least $\gamma>0$. Therefore the expected regret over the prior is at least $\gamma$ times the frequentist regret in Theorem~\ref{thm:alpha_less_than_1_example}, which is linear. 

%% file: proof_convert_distribution_copy.tex
	\section{Proof of Lemma~\ref{lem:convert_distribution}}
	\label{sec:proof_convert_distribution}

%
%

To convert between $D_{\alpha}(\Pi_t, Q_t)$ and $D_{\alpha}(\overline{\Pi_t } , \overline{Q_t})$ we first prove the following lemma:
	\begin{lemma}[Jensen's Inequality]
		\label{lem:Jensen}
		Let $f: \mathcal{R}^2 \rightarrow \mathcal{R}$ be a convex function. Let $P: \mathcal{R}^k \rightarrow \mathcal{R}$ and $Q: \mathcal{R}^k \rightarrow \mathcal{R}$ be 2 functions. Let $S$ is a subset of  $R^k$, the domain of $x$ and $|S|$ denote the volume of $S$. Then
		\begin{align}
		\label{eq:Jensen}
		&\frac{1}{|S|} \int_S  f(P(x),Q(x)) dx  \nonumber\\&\geq f \left (  \frac{1}{|S|} \int_S P(x) dx , \frac{1}{|S|} \int_S Q(x) dx \right )\;.
		\end{align}
	\end{lemma}
	\begin{proof}
		The multivariate Jensen's Inequality states that if $X$ is a n-dimensional random vector and $f: \mathcal{R}^n \rightarrow \mathcal{R}$ is a convex function then 
		\begin{align*}
		 \mathbb{E}(f(X)) \geq f(\mathbb{E}(X))\;.
		\end{align*}
		To use the multivariate Jensen's Inequality we define the $2$-dimensional random vector $X: S \rightarrow \mathcal{R}^2$ by $X(x)  := (P(x), Q(x))$ and a probability distribution over $S$ such that for all $x \in S$: 
		$\mathbb{P}(x) = \frac{1}{|S|}$. 

		Then the left-hand side of Eq.~\ref{eq:Jensen} becomes $\mathbb{E}(f(X))$, while the right-hand side becomes $ f(\mathbb{E}(X))$, and Eq. \ref{eq:Jensen} follows from the multivariate Jensen's Inequality. 
		\end{proof}

	Now we will prove Lemma~\ref{lem:convert_distribution}. 
	\begin{proof}[Proof of Lemma~\ref{lem:convert_distribution}]
		Since $D_{\alpha}(p,q)$ is convex \citep{Cichocki2010FamiliesOA}, the following functions:
	\begin{align*}
	f(p, q) &= q\log\frac{q}{p}, \\
	f(p, q) &= p\log\frac{p}{q}, \\
	f(p, q) &= \frac{p^{\alpha}q^{1-\alpha}} {\alpha ( \alpha-1)}
	\end{align*}
	are convex, and we can apply Lemma~\ref{lem:Jensen}: 
	\begin{itemize}
		\item When $\alpha = 0$:
		\begin{align*}
		&D_{\alpha}(\Pi_t,Q_t) \\
		=& \int q_t(m) \log \frac{q_t(m)}{\pi_t(m)} dm \\
		=& \sum_{i} \int_{\Omega_i} q_t(m) \log \frac{q_t(m)}{\pi_t(m)} dm \\
		\geq& \sum_{i} |\Omega_i| \frac{1}{|\Omega_i|} \int_{\Omega_i} q_t(m) dm \log \frac{\frac{1}{|\Omega_i|}\int_{\Omega_i} q_t(m) dm}{\frac{1}{|\Omega_i|}\int_{\Omega_i} \pi_t(m) dm}  \text{ by applying Lemma~\ref{lem:Jensen}} \\
		=&  \sum_{i}  Q_{t,i} \log \frac{Q_{t,i}}{\Pi_{t,i}}   \\
		=& D_{\alpha} (\overline{\Pi_t}, \overline{Q_t})\;.
		\end{align*}
			\item When $\alpha = 1$:
			\begin{align*}
			&D_{\alpha}(\Pi_t,Q_t) \\
			=& \int \pi_t(m) \log \frac{\pi_t(m)}{q_t(m)} dm \\
			=& \sum_{i} \int_{\Omega_i} \pi_t(m) \log \frac{\pi_t(m)}{q_t(m)} dm \\
			\geq& \sum_{i} |\Omega_i| \frac{1}{|\Omega_i|} \int_{\Omega_i} \pi_t(m) dm \log \frac{\frac{1}{|\Omega_i|}\int_{\Omega_i} \pi_t(m) dm}{\frac{1}{|\Omega_i|}\int_{\Omega_i} q_t(m) dm}  \text{ by applying Lemma~\ref{lem:Jensen}} \\
			=&  \sum_{i}  \Pi_{t,i} \log \frac{\Pi_{t,i}}{Q_{t,i}}   \\
			=& D_{\alpha} (\overline{\Pi_t}, \overline{Q_t})\;.
			\end{align*}
		\item When $\alpha \neq 0, \alpha \neq 1$: 
		\begin{align*}
		&D_{\alpha}(\Pi_t,Q_t) \\
		=&  \int \frac{\pi(x)^{\alpha}q(x)^{1-\alpha} - 1 }{-\alpha(1-\alpha)} dx\\
		=& \frac{-1}{\alpha(\alpha -1) } +  \sum_{i} \int_{\Omega_i} \frac{\pi(x)^{\alpha}q(x)^{1-\alpha} }{\alpha(\alpha-1)} dx  \\
		\geq& \frac{-1}{\alpha(\alpha -1) }  + \sum_{i}  |\Omega_i| \frac{ (\frac{\Pi_{t,i}}{|\Omega_i|})^{\alpha} (\frac{Q_{t,i}}{|\Omega_i|})^{1-\alpha}} {\alpha(\alpha-1)} \text{ by applying Lemma~\ref{lem:Jensen}} \\
		=& \frac{-1}{\alpha(\alpha -1) }  + \sum_{i} \frac{\Pi_{t,i}^{\alpha} Q_{t,i}^{1-\alpha} } {\alpha(\alpha-1)}  \\
		=& D_{\alpha} (\overline{\Pi_t}, \overline{Q_t})\;.
		\end{align*}
		\end{itemize}
	\end{proof}
	

%% file: proof_uniform_copy.tex
\section{Proof of Theorem~\ref{thm:uniform_sublinear_regret}}
\label{sec:proof_uniform}

We will prove that the frequentist regret is sub-linear for any $m^*$. If the algorithm has sub-linear frequentist regret for all values $M= m^*$, the Bayesian regret (which is the expected value over $M$) will also be sub-linear.

Without loss of generalization, let arm $1$ be the best arm. From Lemma~\ref{lem:posterior_concentration}, since $\sum_{t=1}^{\infty} p_t = \infty$, we have  for all arms $i$, $\sum_{t=1}^\infty P(A_t = i|H_{t-1}) = \infty$ and therefore with probability $1$:
\begin{align}
\lim_{t \rightarrow \infty}  \Pi_{t,1} = \lim_{t \rightarrow \infty} \mathbb{P}(A^*=1|H_{t-1}) = 1\;, 
\end{align}
which means that the posterior probability that arm $1$ is the best arm converges to $1$. 

We will prove the theorem by proving the following steps: 
\begin{itemize}
    \item In Lemma~\ref{lem:approximation_convergence} we show that if the probability that the posterior chooses the best arm tends to $1$, then the probability that the approximation chooses the best arm also tends to $1$
    \item In Lemma~\ref{lem:sublinear_regret} and Lemma~\ref{lem:sublinear_regret2} we show that if the probability that the approximation chooses the best arm also tends to $1$ almost surely, then it has sub-linear regret with probability $1$. Therefore it has sub-linear regret in expectation over the history.
\end{itemize}

\begin{lemma}
\label{lem:approximation_convergence}
Let $\alpha \leq 0$ and  arm $1$ be the true best arm. Let $\Omega_i =\{m | m_i = max (m_1, ..., m_k)\}$ be the region where arm $i$ is the best arm. If the posterior probability that arm $1$ is the best arm converges to $1$: 
\begin{align*}
    \lim_{t \rightarrow \infty}  \Pi_{t,1} = 1 
\end{align*}
and for all $t\geq 0$:
\begin{align*}
D_{\alpha} (\Pi_t,Q_t) < \epsilon,
\end{align*}
then the sequence $\{Q_{t,1}\}_t$ where $ Q_{t,1} = \int_{\Omega_1} q_t(m) dm$ converges and 
\begin{align*}
     \lim_{t \rightarrow \infty} Q_{t,1} = 1\;.
\end{align*}
\end{lemma}

Next we show that if the approximate distribution concentrates, then the probability that it chooses the wrong arm decreases  as $T$ goes to infinity. 
\begin{lemma}
\label{lem:sublinear_regret}
If 
\begin{align*}
     \lim_{t \rightarrow \infty} Q_{t,1} = 1 
\end{align*}
then 
\begin{align*}
   \lim_{T \rightarrow \infty} \frac{\sum_{t=1}^T (1- Q_{t,1}) }{T}= 0\;.
\end{align*}
\end{lemma}
From Lemma~\ref{lem:approximation_convergence} and Lemma~\ref{lem:sublinear_regret}, since $\lim_{t \rightarrow \infty}  \Pi_{t,1} = 1$ with probability $1$, we have $\lim_{T \rightarrow \infty} \frac{\sum_{t=1}^T (1- Q_{t,1}) }{T}= 0$ with probability $1$. We will now show that the expected regret is sub-linear:
\begin{lemma}
	\label{lem:sublinear_regret2}
	 Let $p_t= o(1)$ be such that $\sum_{t=1}^\infty p_t = \infty$. For any number of arms $ k$, any prior  $\Pi_0$ and any error threshold $\epsilon > 0$, the following algorithm has $o(T)$ regret: at every time-step $t$,
	 \begin{itemize}
	 	\item with probability $1-p_t$, sample from an approximate posterior $Q_t$ such that $\lim_{T \rightarrow \infty} \frac{\sum_{t=1}^T (1- Q_{t,1}) }{T}= 0 $ with probability $1$, and
	 	\item with probability $p_t$, sample an arm uniformly at random.
	 \end{itemize}
\end{lemma}
\subsection{Proof of Lemma~\ref{lem:approximation_convergence}}
\begin{proof}

Let 
$Q_{t,i} = \int_{\Omega_i} q_t(m) dm$ and $\Pi_{t,i} = \int_{\Omega_i} \pi_t(m) dm$\;.
Then 
\begin{align*}
    \lim_{t \rightarrow \infty} \Pi_{t,1} = 1
\end{align*} and we want to show that $\{Q_{t,1}\}_t$ converges and 
\begin{align*}
    \lim_{t \rightarrow \infty} Q_{t,1} = 1\;.
\end{align*} 
Since $D_{\alpha}(\overline{\Pi_t},\overline{Q_t}) < \epsilon$ and $\lim \Pi_{t,1} =1$ we want to show that $\limsup Q_{t,1} =1$. 
By contradiction, assume that: 
\begin{align*}
\limsup  Q_{t,1} = c < 1\;. 
\end{align*}
Then there exists a sub-sequence of $\{Q_{t,1}\}_t$, denoting $Q_{t_1, 1}, Q_{t_2, 1}, ..., Q_{t_n, 1},..$ such that
\begin{align}
\label{eq:lim_subsequence}
\lim_{n \rightarrow \infty} Q_{t_n, 1} = c\;.
\end{align}

which implies 
\begin{align*}
0< 1- c = \lim_{n \rightarrow \infty} \sum_{i=2}^k  Q_{t_n, i} \leq  \sum_{i=2}^k \limsup_{n \rightarrow \infty} Q_{t_n, i}.
\end{align*}
Therefore there exists $j \in [2,k]$ such that: 
\begin{align*}
\limsup_{n \rightarrow \infty} Q_{t_n, j}= d > 0\;.
\end{align*}
Then there exists a sub-sequence of $\{Q_{t_n,j}\}_n$, denoting $Q_{t_{n_1}, j}, Q_{t_{n_2}, j}, ..., Q_{t_{n_m}, j},..$ such that
\begin{align*}
\lim_{m \rightarrow \infty} Q_{t_{n_m}, j} = d\;.
\end{align*}
We consider the 2 cases: 
\begin{itemize}
	\item When $\alpha = 0$:
	\begin{align*}
	D_{\alpha}(\overline{\Pi}_t,\overline{Q}_t) = \sum_{i=1}^k Q_{t,i} \log \frac{Q_{t,i}}{\Pi_{t,i}}\;.
	\end{align*}
	
	Then we have:
	\begin{align*}
	\epsilon = &\lim_{m \rightarrow \infty} D_{\alpha}(\overline{\Pi}_{t_{n_m}}, \overline{Q}_{t_{n_m}}) \\
	\geq & \lim_{m \rightarrow \infty}  Q_{t_{n_m},1} \log \frac{Q_{t_{n_m},1}}{\Pi_{t_{n_m},1}}  +\lim_{m \rightarrow \infty}  Q_{t_{n_m},j} \log \frac{Q_{t_{n_m},j}}{\Pi_{t_{n_m},j}} \\
	= &  c\log\frac{c}{1} + d\log \frac{d}{0} \\
	= & \infty \text{ since } d > 0, 
	\end{align*}
	which is a contradiction. Therefore $c = 1$. 
	\item When $\alpha < 0$: 
	
	\begin{align*}
	&D_{\alpha}(\overline{\Pi}_t,\overline{Q}_t) 
	= \frac{ \sum_{i=1}^k \Pi_{t,i}^{\alpha} Q_{t,i}^{1-\alpha} - 1} {\alpha(\alpha-1)}\;.
	\end{align*}

	Then we have: 
	\begin{align*}
	\epsilon = &\lim_{m \rightarrow \infty} D_{\alpha}(\overline{\Pi}_{t_{n_m}}, \overline{Q}_{t_{n_m}}) \\
	\geq & \lim_{m \rightarrow \infty}  \frac{\Pi_{t_{n_m},1}^{\alpha} Q_{t_{n_m},1}^{1-\alpha} + \Pi_{t_{n_m},j}^{\alpha} Q_{t_{n_m},j}^{1-\alpha}  - 1} {\alpha(\alpha-1)}  \\
	= &  \frac{1^{\alpha} c^{1-\alpha} + \frac{d^{1-\alpha}}{(0)^{-\alpha}}  - 1} {\alpha(\alpha-1)} \\
	= & \infty, \text{ since } d > 0 \text{ and } \alpha < 0,
	\end{align*}
	which is a contradiction. Therefore $c = 1$.

\end{itemize}
	Similarly we will show that: 
	\begin{align*}
	\liminf Q_{t,1} = 1\;.
	\end{align*}
	By contradiction, assume that: 
	\begin{align*}
	\liminf  Q_{t,1} = c' < 1\;. 
	\end{align*}
	Then there exists a sub-sequence of $\{Q_{t,1}\}_t$, denoting $Q_{t_1, 1}, Q_{t_2, 1}, ..., Q_{t_{n'}, 1},..$ such that
	\begin{align*}
	\lim_{n \rightarrow \infty} Q_{t_{n'}, 1} = c'\;.
	\end{align*}
	Using the same argument following Eq.~\ref{eq:lim_subsequence} we will have $c'=1$. 
	Since $\liminf Q_{t,1} = \limsup Q_{t,1} = 1$, we have that $\{Q_{t,1}\}_t$ converges and
	\begin{align*}
	\lim Q_{t,1} = 1\;.
	\end{align*}
\end{proof}
\subsection{Proof for Lemma~\ref{lem:sublinear_regret}}
For simplicity let $x_t$ denote $1- Q_{t,1}$. We want to show that if a sequence $\{x_t\}$ satisfies $x_t \geq 0 ~\forall t$ and: 
\begin{align*}
     \lim_{t \rightarrow \infty} x_t = 0,
\end{align*}
then 
\begin{align*}
  \lim_{T \rightarrow \infty} S_T = 0,
\end{align*}
where 
\begin{align*}
S_T= \frac{\sum_{t=1}^T x_t }{T}\;. 
\end{align*}

Since $ \lim_{t \rightarrow \infty} x_t = 0$ and $x_t \geq 0~\forall t$, for any $\epsilon > 0$ there exists $T_0$ such that for all $t > T_0$:
\begin{align*}
    x_t < \frac{\epsilon}{2}\;.
\end{align*}
Then for all $T > T_0$: 
\begin{align*}
    S_T & = \frac{x_1 + ... + x_{T_0}}{T} + \frac{x_{T_0+1} + ... + x_T}{T} \\
    & \leq \frac{x_1 + ... + x_{T_0}}{T} + \frac{\frac{\epsilon}{2} T}{T} \\
    & \leq \frac{x_1 + ... + x_{T_0}}{T} + \frac{\epsilon}{2}\;.
\end{align*}
Choose $T_1$ large enough such that $ \frac{x_1 + ... + x_{T_0}}{T_1} < \frac{\epsilon}{2}$. Let $T_2 = \max(T_0, T_1)$. Then for all $T > T_2$: 
\begin{align*}
    S_T = \frac{x_1 + ... + x_{T_0}}{T} + \frac{\epsilon}{2} < \frac{\epsilon}{2} + \frac{\epsilon}{2} = \epsilon\;.
\end{align*}
Therefore for any $\epsilon > 0$, there exists $T_2$ such that for all $T > T_2$, $S_T < \epsilon$. Since $S_T \geq 0~\forall T$, we have: 
\begin{align*}
    \lim_{T \rightarrow \infty} S_T = 0\;.
\end{align*}
\subsection{Proof of Lemma~\ref{lem:sublinear_regret2}}
Without loss of generalization, let arm $1$ be the true best arm. Let $\Delta = m^*_1- \max(m^*_2, ..., m^*_k)$ be the gap between the highest mean $m^*_1$ and the next highest mean of the arms. 

Since $p_t = o(1)$, $\sum_{t=1}^T p_t $ is $o(T)$. Therefore the regret from the uniform sampling steps is $o(T)$. 

Since $1-Q_{t,1}$ is the probability of choosing a sub-optimal arm by sampling from $Q_t$, the regret of sampling from  $Q_t$ is upper bounded by:
\begin{align*}
\mathbb{E}\sum_{t=1}^T \Delta (1-Q_{t,1})\;.
\end{align*}
Since $   \lim_{T \rightarrow \infty} \frac{\sum_{t=1}^T (1- Q_{t,1}) }{T}= 0 $ with probability $1$, we have
\begin{align*}
\lim_{T \rightarrow \infty} \frac{\sum_{t=1}^T \Delta (1- Q_{t,1}) }{T}= 0
\end{align*}
with probability $1$. Therefore 
\begin{align*}
\lim_{T \rightarrow \infty} \mathbb{E}\frac{\sum_{t=1}^T \Delta (1- Q_{t,1}) }{T}= 0,
\end{align*}
which means that the regret of sampling from $Q_t$ is sub-linear. 
Since both the expected regrets of the uniform sampling steps and of sampling from $Q_t$ are sub-linear, the total expected regret is sub-linear.      

%% file: related_works_appendix.tex
\section{Ensemble Sampling and Uniform Exploration}
\label{sec:related_works_appendix}
To the best of our knowledge, \citep{NIPS2017_6918} is the only work that provides a theoretical analysis of Thompson sampling when the sampling step is approximate. \citet{NIPS2017_6918} propose an approximate sampling method  called Ensemble sampling where they maintain a set of $\mathcal{M}$  models to approximate the posterior, and analyze its regret for linear contextual bandits.  When the model is a $k$-armed bandit, the regret bound is as follow: 
\begin{lemma}[implied by \citep{NIPS2017_6918}]
	Let $\pi^{TS}$ and $\pi^{ES}$ denote the exact Thompson sampling and Ensemble sampling policies. Let $\Delta = max(m^*_1, ..., m^*_k) - \min(m^*_1, ..., m^*_k)$. 
	For all $\epsilon > 0$, if 
	\begin{align*}
	\mathcal{M}  \geq \frac{2k}{\epsilon^2} \log \frac{2kT}{\epsilon^2\delta}, 
	\end{align*}
	then 
	\begin{align}
	\label{eq:ensemble_sampling}
	\mathrm{Regret} (T, \pi^{ES}) \leq \mathrm{Regret}(T,\pi^{TS} ) + \epsilon \Delta T + \delta \Delta T
	\end{align}
\end{lemma}

\citet{NIPS2017_6918} prove the regret bound by only using the following property of the Ensemble sampling method: at time $t$, with probability $1-\delta$, Ensemble sampling satisfies the following constraint: 
\begin{align}
\label{eq:KL_bar_constraint}
\mathrm{KL}(\overline{Q}_t, \overline{\Pi}_t) < \epsilon^2,
\end{align}
where $\epsilon$ is a constant if $\mathcal{M}$  is  $\Theta(\log(T))$. If $\epsilon$ is a constant the regret will be linear because of the term $\epsilon \Delta T$.

At time $t$, with probability $1-\delta$, $\mathrm{KL}(\overline{Q}_t, \overline{\Pi}_t) < \epsilon^2$. The first 2 terms in the right hand side of Eq.~\ref{eq:ensemble_sampling} comes from the time-steps when $\mathrm{KL}(\overline{Q}_t, \overline{\Pi}_t) < \epsilon^2$, and the last term comes from the other case with probability $\delta$.


Theorem~\ref{thm:uniform_sublinear_regret} shows that applying an uniform sampling step will make the posterior concentrate. 
Moreover, Lemma~\ref{lem:approximation_convergence} implies that if Eq.~\ref{eq:KL_bar_constraint} is satisfied at a subset of times $\mathcal{T}_0 \subseteq [0,1,..., T]$,  the approximation $Q_t$ will also concentrate when $t \in \mathcal{T}_0$. Therefore the regret from the time-steps in $\mathcal{T}_0$ will be sub-linear in $\mathcal{T}_0$, which is sub-linear in $T$.

So if we want to maintain a small number of models $M = \Theta(\log(T))$ and achieve sub-linear regret, we can apply Theorem~\ref{thm:uniform_sublinear_regret} as follow.  First we choose $\delta$ to be small such that the last term in Eq.~\ref{eq:ensemble_sampling} $\delta \Delta T $ is $o(T)$. Then we apply the uniform sampling step as shown in Theorem~\ref{thm:uniform_sublinear_regret}, so that the first 2 terms in the right hand side of Eq.~\ref{eq:ensemble_sampling} become sub-linear. We can then achieve sub-linear regret with Ensemble sampling with a $\Theta(\log T)$ number of models. 

%% file: main.bbl
\begin{thebibliography}{23}
\providecommand{\natexlab}[1]{#1}
\providecommand{\url}[1]{\texttt{#1}}
\expandafter\ifx\csname urlstyle\endcsname\relax
  \providecommand{\doi}[1]{doi: #1}\else
  \providecommand{\doi}{doi: \begingroup \urlstyle{rm}\Url}\fi

\bibitem[Agrawal \& Goyal(2013)Agrawal and Goyal]{pmlr-v31-agrawal13a}
Agrawal, S. and Goyal, N.
\newblock Further optimal regret bounds for {T}hompson sampling.
\newblock In \emph{Proceedings of the Sixteenth International Conference on
  Artificial Intelligence and Statistics (AISTATS 2013)}, volume~31 of
  \emph{Proceedings of Machine Learning Research}, pp.\  99--107. PMLR, 2013.

\bibitem[Andrieu et~al.(2003)Andrieu, de~Freitas, Doucet, and
  Jordan]{Andrieu2003}
Andrieu, C., de~Freitas, N., Doucet, A., and Jordan, M.~I.
\newblock An introduction to {MCMC} for machine learning.
\newblock \emph{Machine Learning}, 50\penalty0 (1):\penalty0 5--43, 2003.
\newblock ISSN 1573-0565.
\newblock \doi{10.1023/A:1020281327116}.

\bibitem[Blei et~al.(2017)Blei, Kucukelbir, and
  McAuliffe]{doi:10.1080/01621459.2017.1285773}
Blei, D.~M., Kucukelbir, A., and McAuliffe, J.~D.
\newblock Variational inference: A review for statisticians.
\newblock \emph{Journal of the American Statistical Association}, 112\penalty0
  (518):\penalty0 859--877, 2017.
\newblock \doi{10.1080/01621459.2017.1285773}.

\bibitem[Carpenter et~al.(2017)Carpenter, Gelman, Hoffman, Lee, Goodrich,
  Betancourt, Brubaker, Guo, Li, and Riddell]{stan:2017}
Carpenter, B., Gelman, A., Hoffman, M., Lee, D., Goodrich, B., Betancourt, M.,
  Brubaker, M., Guo, J., Li, P., and Riddell, A.
\newblock Stan: {A} probabilistic programming language.
\newblock \emph{Journal of Statistical Software}, 76\penalty0 (1), 2017.

\bibitem[Cichocki \& Amari(2010)Cichocki and Amari]{Cichocki2010FamiliesOA}
Cichocki, A. and Amari, S.
\newblock Families of alpha- beta- and gamma- divergences: Flexible and robust
  measures of similarities.
\newblock \emph{Entropy}, 12:\penalty0 1532--1568, 2010.

\bibitem[Doucet \& Johansen(2009)Doucet and Johansen]{Doucet11atutorial}
Doucet, A. and Johansen, A.
\newblock A tutorial on particle filtering and smoothing: Fifteen years later.
\newblock \emph{Handbook of Nonlinear Filtering}, 12:\penalty0 656--704, 2009.

\bibitem[Gopalan et~al.(2014)Gopalan, Mannor, and Mansour]{2013arXiv1311.0466G}
Gopalan, A., Mannor, S., and Mansour, Y.
\newblock Thompson sampling for complex online problems.
\newblock In \emph{Proceedings of the 31st International Conference on Machine
  Learning}, volume~32 of \emph{Proceedings of Machine Learning Research}, pp.\
   100--108, Bejing, China, 22--24 Jun 2014. PMLR.

\bibitem[Kawale et~al.(2015)Kawale, Bui, Kveton, Tran-Thanh, and
  Chawla]{NIPS2015_5985}
Kawale, J., Bui, H.~H., Kveton, B., Tran-Thanh, L., and Chawla, S.
\newblock Efficient {T}hompson sampling for online matrix-factorization
  recommendation.
\newblock In \emph{Advances in Neural Information Processing Systems 28 (NIPS
  2015)}, pp.\  1297--1305. Curran Associates, Inc., 2015.

\bibitem[Kveton et~al.(2019)Kveton, Szepesvari, Vaswani, Wen, Lattimore, and
  Ghavamzadeh]{pmlr-v97-kveton19a}
Kveton, B., Szepesvari, C., Vaswani, S., Wen, Z., Lattimore, T., and
  Ghavamzadeh, M.
\newblock Garbage in, reward out: Bootstrapping exploration in multi-armed
  bandits.
\newblock In \emph{Proceedings of the 36th International Conference on Machine
  Learning}, volume~97 of \emph{Proceedings of Machine Learning Research}, pp.\
   3601--3610, Long Beach, California, USA, 09--15 Jun 2019. PMLR.

\bibitem[Liu \& Li(2016)Liu and Li]{DBLP:journals/corr/LiuL15c}
Liu, C.-Y. and Li, L.
\newblock On the prior sensitivity of thompson sampling.
\newblock In \emph{Algorithmic Learning Theory}, pp.\  321--336, Cham, 2016.
  Springer International Publishing.
\newblock ISBN 978-3-319-46379-7.

\bibitem[Lu \& Van~Roy(2017)Lu and Van~Roy]{NIPS2017_6918}
Lu, X. and Van~Roy, B.
\newblock Ensemble sampling.
\newblock In \emph{Advances in Neural Information Processing Systems 30 (NIPS
  2017)}, pp.\  3260--3268. Curran Associates, Inc., 2017.

\bibitem[Minka(2005)]{divergence-measures-and-message-passing}
Minka, T.
\newblock Divergence measures and message passing.
\newblock Technical Report MSR-TR-2005-173, January 2005.

\bibitem[Minka et~al.(2018)Minka, Winn, Guiver, Zaykov, Fabian, and
  Bronskill]{InferNET18}
Minka, T., Winn, J., Guiver, J., Zaykov, Y., Fabian, D., and Bronskill, J.
\newblock /{Infer.{NET} 0.3}, 2018.
\newblock Microsoft Research Cambridge. http://dotnet.github.io/infer.

\bibitem[Ordentlich \& Weinberger(2004)Ordentlich and
  Weinberger]{Ordentlich2004ADD}
Ordentlich, E. and Weinberger, M.~J.
\newblock A distribution dependent refinement of {P}insker's inequality.
\newblock \emph{International Symposium on Information Theory, 2004. ISIT 2004.
  Proceedings.}, pp.\  29--, 2004.

\bibitem[Osband et~al.(2016)Osband, Van~Roy, and
  Wen]{Osband:2016:GEV:3045390.3045641}
Osband, I., Van~Roy, B., and Wen, Z.
\newblock Generalization and exploration via randomized value functions.
\newblock In \emph{Proceedings of the 33rd International Conference on
  International Conference on Machine Learning - Volume 48}, ICML'16, pp.\
  2377--2386. JMLR.org, 2016.

\bibitem[Qin et~al.(2017)Qin, Klabjan, and Russo]{NIPS2017_7122}
Qin, C., Klabjan, D., and Russo, D.
\newblock Improving the expected improvement algorithm.
\newblock In \emph{Advances in Neural Information Processing Systems 30}, pp.\
  5381--5391. Curran Associates, Inc., 2017.

\bibitem[Riquelme et~al.(2018)Riquelme, Tucker, and
  Snoek]{DBLP:journals/corr/abs-1802-09127}
Riquelme, C., Tucker, G., and Snoek, J.
\newblock Deep {B}ayesian bandits showdown: An empirical comparison of bayesian
  deep networks for thompson sampling.
\newblock In \emph{International Conference on Learning Representations (ICLR
  2018)}, 2018.

\bibitem[Russo(2016)]{DBLP:journals/corr/Russo16}
Russo, D.
\newblock Simple {B}ayesian algorithms for best arm identification.
\newblock In \emph{29th Annual Conference on Learning Theory (COLT 2016)},
  volume~49 of \emph{Proceedings of Machine Learning Research}, pp.\
  1417--1418. PMLR, 2016.

\bibitem[Russo \& Roy(2016)Russo and Roy]{JMLR:v17:14-087}
Russo, D. and Roy, B.~V.
\newblock An information-theoretic analysis of {T}hompson sampling.
\newblock \emph{Journal of Machine Learning Research}, 17\penalty0
  (68):\penalty0 1--30, 2016.

\bibitem[Russo et~al.(2018)Russo, Roy, Kazerouni, Osband, and Wen]{MAL-070}
Russo, D.~J., Roy, B.~V., Kazerouni, A., Osband, I., and Wen, Z.
\newblock A tutorial on {T}hompson sampling.
\newblock \emph{Foundations and Trends® in Machine Learning}, 11\penalty0
  (1):\penalty0 1--96, 2018.
\newblock ISSN 1935-8237.
\newblock \doi{10.1561/2200000070}.

\bibitem[Salvatier et~al.(2016)Salvatier, Wiecki, and
  Fonnesbeck]{Salvatier2016ProbabilisticPI}
Salvatier, J., Wiecki, T.~V., and Fonnesbeck, C.
\newblock Probabilistic programming in {Python using PyMC3}.
\newblock \emph{PeerJ Computer Science}, 2:\penalty0 e55, 2016.

\bibitem[Tran et~al.(2016)Tran, Kucukelbir, Dieng, Rudolph, Liang, and
  Blei]{tran2016edward}
Tran, D., Kucukelbir, A., Dieng, A.~B., Rudolph, M., Liang, D., and Blei, D.~M.
\newblock {Edward: A library for probabilistic modeling, inference, and
  criticism}.
\newblock \emph{arXiv preprint arXiv:1610.09787}, 2016.

\bibitem[Urteaga \& Wiggins(2018)Urteaga and Wiggins]{pmlr-v84-urteaga18a}
Urteaga, I. and Wiggins, C.
\newblock Variational inference for the multi-armed contextual bandit.
\newblock In \emph{Proceedings of the Twenty-First International Conference on
  Artificial Intelligence and Statistics (AISTATS 2018)}, volume~84 of
  \emph{Proceedings of Machine Learning Research}, pp.\  698--706. PMLR, 2018.

\end{thebibliography}
